\documentclass[nohyperref]{article}

\usepackage{microtype}
\usepackage{graphicx}
\usepackage{graphbox}
\usepackage{subfigure}
\usepackage{booktabs} % for professional tables
\usepackage{lipsum}
\usepackage{xcolor}

\usepackage{hyperref}

\usepackage[accepted]{icml2022}

\usepackage{amsmath,bm}
\usepackage{amssymb}
\usepackage{mathtools}
\usepackage{amsthm}

\usepackage[capitalize,noabbrev]{cleveref}

\theoremstyle{plain}
\newtheorem{theorem}{Theorem}[section]
\newtheorem{proposition}[theorem]{Proposition}
\newtheorem{lemma}[theorem]{Lemma}

\theoremstyle{definition}
\newtheorem{definition}[theorem]{Definition}

\theoremstyle{remark}

\newtheorem*{theorem*}{Theorem}
\newtheorem*{lemma*}{Lemma}

\usepackage[textsize=tiny]{todonotes}
\definecolor{myred}{HTML}{E6194B}

\icmltitlerunning{Optimistic Linear Support and SFs as a Basis for Optimal Policy Transfer}

\newcommand{\vect}[1]{\boldsymbol{\mathbf{#1}}}
\newcommand{\Ex}{\mathbb{E}}

\newcommand{\pigpi}{\pi^{\text{GPI}}}
\newcommand{\pismp}{\pi^{\text{SMP}}}
\newcommand{\ccs}{\mathrm{CCS}}
\DeclareMathOperator*{\argmax}{arg\,max}
\DeclareMathOperator*{\argmin}{arg\,min}

\begin{document}

\twocolumn[
\icmltitle{Optimistic Linear Support and Successor Features\\ as a Basis for Optimal Policy Transfer}




\begin{icmlauthorlist}
\icmlauthor{Lucas N. Alegre}{inf}
\icmlauthor{Ana L. C. Bazzan}{inf}
\icmlauthor{Bruno C. da Silva}{cics}
\end{icmlauthorlist}

\icmlaffiliation{inf}{Institute of Informatics, Federal University of Rio Grande do Sul, Porto Alegre, RS, Brazil}
\icmlaffiliation{cics}{University of Massachusetts Amherst, MA}

\icmlcorrespondingauthor{Lucas N. Alegre}{lnalegre@inf.ufrgs.br}

\icmlkeywords{multi-objective reinforcement learning, successor features, policy transfer, transfer learning, generalized policy improvement}

\vskip 0.3in
]



\printAffiliationsAndNotice{}  % leave blank if no need to mention equal contribution

% either one at a time, or multiple (possibly conflicting) objectives simultaneously. The latter setting has been extensively studied in the multi-objective RL literature. 

% In many real-world applications, reinforcement learning (RL) agents might be tasked with optimizing various objectives---either one at a time, or multiple objectives simultaneously. These settings have been extensively studied, respectively, in the transfer learning and multi-objective RL literature. 

\begin{abstract}
In many real-world applications, reinforcement learning (RL) agents might have to solve multiple tasks, each one typically modeled via a reward function.
If reward functions are expressed linearly, and the agent has previously learned a set of policies for different tasks, \textit{successor features} (SFs) can be exploited to combine such policies and identify reasonable solutions for new problems. 
However, the identified solutions are not guaranteed to be optimal. 
We introduce a novel algorithm that addresses this limitation.
It allows RL agents to combine existing policies and \textit{directly} identify optimal policies for arbitrary new problems, without requiring any further interactions with the environment. 
We first show (under mild assumptions) that the transfer learning problem tackled by SFs is equivalent to the problem of learning to optimize multiple objectives in RL.
We then introduce an SF-based extension of the \textit{Optimistic Linear Support} algorithm to learn a set of policies whose SFs form a convex coverage set.
We prove that policies in this set can be combined via generalized policy improvement to construct optimal behaviors for any new linearly-expressible tasks, without requiring any additional training samples.
We empirically show that our method outperforms state-of-the-art competing algorithms both in discrete and continuous domains under value function approximation.
\end{abstract}

\section{Introduction}
\label{sec:intro}

Reinforcement learning (RL) has been successfully applied to solve many complex problems in a wide range of domains \cite{Silver+2017,Vinyals+2019,Bellemare+2020}.
However, real-world problems often require optimizing  multiple tasks or identifying solutions that optimize possibly conflicting objectives, such as efficiency, energy, and safety. Typically, these tasks or objectives are encoded via separate reward functions.
A simple solution to solve problems in this setting is to combine the different objectives of the agent according to some weighting scheme, thereby producing a single scalar reward. This allows for standard RL algorithms to be used, but introduces the question of how to balance the relative importance of each objective \citep{Vamplew+2021arxiv}. 
As an alternative, one can identify \textit{multiple} decision-making policies---each one specialized in solving a particular objective or a particular weighted combination of objectives. This is useful because it allows the user of the algorithm to select a policy based on the particular task they are tackling, or based on their current relative preferences over objectives.

Consider an agent that needs to solve multiple tasks whose reward functions can be linearly-expressed (i.e., represented as the weighted sum of reward features and reward weights). In this setting, \textit{successor features} (SFs) \cite{Barreto+2017} have been shown to be a promising approach capable of composing previously-learned policies to solve novel tasks \cite{Barreto+2018,Borsa+2019,Gimelfarb+2021,Nemecek&Parr2021}.
SFs, in particular, allow for the policy evaluation and policy improvement steps (which underlie most RL algorithms) to be extended to the case where the agent needs \textit{(i)} to evaluate a policy on multiple tasks; or \textit{(ii)} to construct a policy (appropriate for solving a novel task) by improving upon an existing set of policies. These processes are known, respectively, as \textit{generalized policy evaluation} (GPE) and \textit{generalized policy improvement} (GPI) \cite{Barreto+2020}.
This is relevant because, given the weights describing a new reward function (or task), GPI can be used to construct a policy that performs better than the existing, previously-learned ones known to the agent.
However, the solutions identified by GPI are \textit{not} guaranteed to be optimal policies. \textbf{This leads to an important open problem: how to construct a set of policies, such that combining them directly leads to the optimal policy for \textit{any} novel linearly-expressible tasks?}

A closely-related problem to the one of solving multiple tasks is that of simultaneously optimizing multiple objectives. This problem has been extensively studied in the multi-objective RL (MORL) literature \cite{vanMoffaert&Nowe2014,Abels+2019,Yang+2019,Hayes+2022}.
When an agent's relative preferences over objectives can be expressed linearly, an optimal solution corresponds to a \textit{convex coverage set} (CCS). This means, in particular, that given \textit{any} new linear preferences, there exists at least one corresponding optimal policy in such a set.
The \textit{Optimistic Linear Support} (OLS) algorithm \cite{Roijers+2015,Mossalam+2016} is a state-of-the-art technique for constructing a CCS by iteratively learning policies specialized in optimizing different linear preferences. %over an agent's objectives.

Even though SFs and MORL address related problems (i.e., optimizing multiple tasks, or optimizing multiple objectives), they have typically been investigated independently. \textbf{In this paper we show that it is possible to combine and extend theoretical guarantees from each of these optimization frameworks \textit{to directly identify optimal policies for any new tasks}. We refer to this as \textit{optimal policy transfer}. }
We first show (under mild assumptions) that the transfer learning problem tackled by SFs is equivalent to the problem of learning policies that solve a multi-objective Markov decision process (MOMDP) under linear preferences.
We then introduce an SF-based extension of the OLS algorithm (SFOLS) capable of identifying which tasks to solve so that the SFs corresponding to their policies form a CCS. Importantly, we prove that policies in this set can be combined via GPI to directly construct optimal behaviors for \textit{any} new linearly-expressible tasks, without requiring the agent to collect any additional training samples.
Furthermore, we also prove that when an \textit{incomplete} CCS is available, it is possible to bound the gap between an optimal solution and solutions derived via GPI.
Additionally, in Appendix~\ref{sec:skill-discovery} we show that SFOLS can be used to solve a recently-proposed, important open problem in the area of optimal skill discovery in the unsupervised setting \cite{Eysenbach+2022}. Even though SFOLS was not designed to this end, it can nonetheless be used to identify optimal solutions to this open problem.
Finally, we empirically show that SFOLS outperforms state-of-the-art competing algorithms in the optimal policy transfer setting, both in discrete and continuous domains under value function approximation, as well as in a zero-shot/lifelong learning setting.

\section{Background}

Here, we discuss important definitions (and corresponding notation) associated with RL, SFs, GPI, and MORL.

\subsection{Reinforcement Learning}
\label{sec:rl}
A RL problem \cite{Sutton&Barto2018} is typically modeled as a \textit{Markov decision process} (MDP). An MDP can be defined as a tuple $M \equiv (\mathcal{S},\mathcal{A},p,r,\mu,\gamma)$, where $\mathcal{S}$ is a state space,  $\mathcal{A}$ is an action space, $p(\cdot|s,a)$ describes the distribution over next states given that the agent executed action $a$ in state $s$, $r : \mathcal{S} \times \mathcal{A} \times \mathcal{S} \mapsto \mathbb{R}$ is a reward function, $\mu$ is an initial state distribution, and $\gamma \in [0,1)$ is a discounting factor.
Let $S_t$, $A_t$ and $R_t = r(S_t,A_t,S_{t+1})$ be random variables corresponding to the state, action, and reward, respectively, at time step $t$. 
The goal of the agent is to learn a policy $\pi:\mathcal{S} \mapsto \mathcal{A}$ that maximizes the expected discounted sum of rewards (\textit{return}) $G_t = \sum_{i=0}^{\infty} \gamma^{i} R_{t+i}$.
The \textit{action-value function} of a policy $\pi$ is defined as $q^\pi(s,a) \equiv \Ex_\pi [G_t  |  S_t = s, A_t = a]$, where $\Ex_\pi [\cdot]$ denotes the expectation over trajectories induced by $\pi$. 
Given $q^\pi$, one can define a \textit{greedy} policy $\pi'(s) \in \argmax\nolimits_a q^\pi(s,a)$. 
It is guaranteed that $q^{\pi'}(s,a) \geq q^\pi(s,a), \forall (s,a) \in \mathcal{S} \times \mathcal{A}$. 
The processes of computing $q^\pi$ and $\pi'$ are known, respectively, as the \textit{policy evaluation} and \textit{policy improvement} steps. Under certain conditions, repeatedly executing the policy evaluating and improvement steps leads to an optimal policy $\pi^*(s) \in \argmax\nolimits_a q^*(s,a)$ \citep{Puterman2005}.

\subsection{Successor Features and GPI}

Let $r_{\vect{w}}$ be a reward function that can be linearly-expressed as $r_{\vect{w}}(s,a,s') = \vect{\phi}(s,a,s')\cdot \vect{w}$, where $\vect{\phi}(s,a,s') \in \mathbb{R}^d$ are reward features and $\vect{w} \in \mathbb{R}^d$ are weights.
The features $\vect{\phi}(s,a,s')$ often represent aspects of the environment that the agent cares about and are related to its objective.
Given a policy $\pi$, we can define the \textit{successor features} (SFs) $\vect{\psi}(s,a) \in \mathbb{R}^d$ for a given state-action pair $(s,a)$ as:
\begin{equation}
\label{eq:psi}
    \vect{\psi}^\pi(s,a) \equiv \Ex_\pi \left[ \sum_{i=0}^{\infty} \gamma^i \vect{\phi}_{t+i}  | S_t=s, A_t=a \right],
\end{equation}
where $\vect{\phi}_t = \vect{\phi}(S_t,A_t,S_{t+1})$.
Given the SFs $\vect{\psi}^{\pi}(s,a)$ associated with a particular policy $\pi$, it is possible to perform policy evaluation and directly compute the action-value function $q^\pi_{\vect{w}}(s,a)$ of $\pi$ under \textit{any} reward function $r_{\vect{w}}$:
\begin{align}
q_{\vect{w}}^\pi(s,a) &= \Ex_\pi \left[ \sum_{i=0}^{\infty} \gamma^i \vect{\phi}_{t+i}\cdot \vect{w} | S_t=s, A_t=a  \right]  \\
&= \vect{\psi}^\pi(s,a)\cdot \vect{w} .
\label{eq:SF_dot_product}
\end{align}
Let $\vect{\psi}^\pi$ be the expected SF vector associated with $\pi$, where the expectation is with respect to the initial state distribution:
$
    \vect{\psi}^\pi = \Ex_{S_0 \sim \mu} \left[ \vect{\psi}^\pi(S_0, \pi(S_0)) \right].
$
Then, the value of a policy $\pi$ under any given reward function $\vect{w}$ can be expressed as  $v^{\pi}_{\vect{w}} = \vect{\psi}^\pi \cdot \vect{w}$.
A key insight is that the definition of SFs corresponds to a form of the Bellman equation, where the features $\vect{\phi}_t$ play the role of rewards. Thus, SFs can be learned through any temporal-difference learning algorithm.

\paragraph{Generalized Policy Evaluation and Improvement.} 
GPI generalizes the policy improvement step (introduced in Section~\ref{sec:rl}) by improving a given policy, tasked with solving a particular task, based on a \textit{set} of action-value functions, instead of a single one.
Assume the agent has access to a set of previously-learned policies $\Pi = \{\pi_i\}_{i=1}^{n}$ and their corresponding SFs, $\Psi = \{\vect{\psi}^{\pi_i}\}_{i=1}^{n}$.
It is possible to evaluate all policies $\pi_i \in \Pi$ under arbitrary reward functions $r_{\vect{w}}$, via Eq.~\eqref{eq:SF_dot_product}: $q_{\vect{w}}^{\pi_i}(s,a)= \vect{\psi}^{\pi_i}(s,a) \cdot \vect{w}$.
This is known as \textit{generalized policy evaluation} (GPE). We define a \textit{GPI policy} as an extension of the standard definition of policy. In particular, it is a \textit{generalized policy} $\pi: \mathcal{S}\times\mathcal{W} \mapsto \mathcal{A}$, defined based on a policy set $\Pi$ and a weight vector $\vect{w}$:
\begin{equation}
\label{eq:sf-gpi}
\pigpi(s; \vect{w}) \in \underset{a \in \mathcal{A}}{\arg\max} \,\, \underset{\pi \in \Pi}{\max} \,\, q^{\pi}_{\vect{w}}(s,a) .
\end{equation}
Let $q^{\text{GPI}}_{\vect{w}}(s,a)$ be the action-value function of policy $\pigpi(\cdot;\vect{w})$. %for any state-action pair $(s,a)$. 
The GPI theorem \cite{Barreto+2017} ensures that $q^{\text{GPI}}_{\vect{w}}(s,a) \geq \max_{\pi \in \Pi} q_{\vect{w}}^{\pi}(s,a)$ for all $(s,a) \in \mathcal{S} \times \mathcal{A}$.
In other words, Eq.~\eqref{eq:sf-gpi} can be used to define a policy that is guaranteed to perform at least as well as any other policies $\pi_i \in \Pi$ in a new task, $\vect{w}$.
Hence, GPI can be seen as a form of \textit{transfer learning} \citep{Taylor&Stone2009}.
The GPI theorem can be extended to the case where we replace $q^{\pi_i}$ with an approximation, $\tilde{q}^{\pi_i}$ \citep{Barreto+2018}.

\subsection{Multi-Objective Reinforcement Learning}

The multi-objective RL setting (MORL) is used to model problems where an agent is tasked with optimizing possibly conflicting objectives, each one modeled via a separate reward function. Formally, MORL problem are modeled as multi-objective MDPs (MOMDPs), which differ from regular MDPs in that the reward function is a vector-valued function $\vect{r}: \mathcal{S} \times \mathcal{A} \times \mathcal{S} \mapsto \mathbb{R}^m$, where $m$ is the number of objectives.
Then, the multi-objective action-value function of a given policy $\pi$ is defined as:
\begin{equation}
     \vect{q}^\pi(s,a) \equiv \Ex_\pi \left[ \sum_{i=0}^{\infty} \gamma^i \vect{R}_{t+i}  | S_t=s, A_t=a \right],
\end{equation}
\noindent where $\vect{q}^\pi(s,a)$ is a $m$-dimensional vector with the $i$-th entry corresponding to the expect return of policy $\pi$ (given the  state-action pair $(s,a)$) under the $i$-th objective.
Let $\vect{v}^\pi \in \mathbb{R}^m$ be the multi-objective value of the policy $\pi$ under the initial state distribution $\mu$:
$
    \vect{v}^\pi = \Ex_{S_0 \sim \mu} \left[ \vect{q}^\pi(S_0, \pi(S_0)) \right],
$
where $v^{\pi}_{i}$ is the value of policy $\pi$ under the $i$-th objective.
Let a \textit{user utility function} (or scalarization function) $u : \mathbb{R}^m \mapsto \mathbb{R}$ be a mapping from the multi-objective value of policy $\pi$, $\vect{v}^\pi$, to a scalar. 
Utility functions often linearly combine the value of a policy under each of the $m$ objectives using a set of weights $\vect{w}$:
$
   u(\vect{v}^\pi, \vect{w})
    %= v_{\vect{w}}^{\pi}
    = \vect{v}^\pi\cdot \vect{w},
$
\noindent where each element of $\vect{w} \in \mathbb{R}^m$ specifies the relative importance of each objective.
For any given constant $\vect{w}$ (i.e., a particular way of weighting objectives), the original MOMDP collapses into an MDP with reward function $r_{\vect{w}}(s,a,s') = \vect{r}(s,a,s') \cdot \vect{w}$.
Let a \textit{Pareto frontier} be a set of nondominated multi-objective value functions $\vect{v}^\pi$: $\mathcal{F} \equiv \{\vect{v}^\pi | \nexists \pi' \mathrm{s.t. }  \vect{v}^{\pi'} \succ_p \vect{v}^{\pi} \}$, where $\succ_p$ is the \textit{Pareto dominance relation} $\vect{v}^\pi \succ_p \vect{v}^{\pi'} \iff (\forall i : v^\pi_i \geq v^{\pi'}_{i}) \land  (\exists i : v^{\pi}_i > v^{\pi'}_{i})$.
We define the optimal solution to a MOMDP as the set of all policies $\pi$ such that $\vect{v}^\pi$ is in the Pareto frontier.
Given a linear utility function $u$, we can define a \textit{convex coverage set} (CCS)  \citep{Roijers+2013} as a finite convex subset of $\mathcal{F}$, such that there exists a policy in the set that is optimal with respect to any linear preference $\vect{w}$. In other words, the CCS is the set of nondominated multi-objective value functions $\vect{v}^\pi$ where the dominance relation is now defined over scalarized values:
\begin{equation}
\label{eq:ccs}
    \ccs {\equiv}\{\vect{v}^\pi \in \mathcal{F} \ | \ \exists \vect{w}\ \text{s.t.}\ \forall \vect{v}^{\pi'} {\in} \mathcal{F}, \vect{v}^{\pi} \cdot \vect{w} {\geq} \vect{v}^{\pi'} \cdot \vect{w} \},
\end{equation}
\noindent where $\vect{w}$ is a vector of weights used by the scalarization function $u$.
The above definition implies that the optimal solution to a MOMDP, under linear preferences, is a finite convex subset of the Pareto frontier. 

\vspace{-0.2cm}
\section{Optimal Policy Transfer}

As mentioned in Section~\ref{sec:intro}, a key contribution of this paper is to combine and extend theoretical guarantees from the SFs and MORL literature to construct new methods capable of directly identifying optimal policies for any new tasks. We refer to this as the optimal policy transfer problem. We first formally define such a problem. Then, we demonstrate how to map any transfer learning problem defined within the SFs framework to an equivalent multi-objective problem modeled as a MOMDP under linear utility functions.
In Section~\ref{sec:method} we derive a principled method with theoretical guarantees for constructing a CCS. These contributions are relevant because, as will be shown, performing GPI over the policies in the CCS is a sufficient condition to ensure that---given any novel linearly-expressible tasks---the resulting policy will be optimal. \textbf{We will show that by mapping SF transfer learning problems to equivalent multi-objective problems, and by exploiting properties of GPI and CCS, we can construct a principle algorithm that achieves our goal of performing optimal policy transfer.}

\subsection{Problem Formulation}
\label{sec:problem}

In the SFs literature, transfer learning is defined as the problem of combining existing policies to identify a (typically sub-optimal, but reasonable) policy for a novel task. Let $\mathcal{M}^\phi$ be the---possibly infinite---set of MDPs associated with all linearly-expressible reward functions:
\begin{equation}
    \label{eq:multi-task-mdp}
    \mathcal{M}^\phi {\equiv} \{ (\mathcal{S}, \mathcal{A}, p, r_{\vect{w}}, \mu, \gamma) | r_{\vect{w}}(s,a,s') {=} \vect{\phi}(s,a,s') {\cdot} \vect{w} \}.
\end{equation}
Assume an agent has learned a set of policies, $\Pi$, for solving some set of tasks, $\mathcal{M} \subset \mathcal{M}^\phi$. By using SF and GPI, it is possible to perform transfer knowledge by composing such policies to construct a new specialized policy for solving a novel new task $M \not\in \mathcal{M}$. 
The question of which set of policies, $\Pi$, should be learned by the agent to facilitate transfer learning is an open problem. Our goal is to construct a policy set $\Pi$ such that the value of the GPI policy $\pigpi$, derived from $\Pi$, is as close as possible to the value of the optimal policy for \textit{any} tasks $\vect{w} \in \mathcal{W}$. That is,
\begin{equation}
\label{eq:problem}
    \argmin_\Pi  \Ex_{\vect{w}\sim\mathcal{W}} \left[ \mathcal{L}(\pigpi, \vect{w}) \right],
\end{equation}
where the expectation is over tasks $\vect{w}$ drawn uniformly at random from the set $\mathcal{W}$ and $\mathcal{L}(\pi, \vect{w}) =  v^{*}_{\vect{w}} - v^\pi_{\vect{w}}$, where $v^{*}_{\vect{w}}$ is the value of the optimal policy under a given reward function $r_{\vect{w}}$: $v_{\vect{w}}^* = \max\nolimits_{\pi} \vect{\psi}^\pi \cdot \vect{w}$.    

Without loss of generality, we consider weight vectors $\vect{w} \in \mathcal{W}$ that induce convex combinations of the features; that is, $\sum_{i} w_i = 1$ and $w_i \geq 0, \forall i$. This is common practice in the MORL literature \cite{Yang+2019}, as it does not alter the optimal policies of the MDPs since optimal policies are invariant with respect to the scale of the rewards.

\subsection{Bridging Successor Features and MORL}

In this section, we show that any transfer learning problem within the SF framework can be mapped into an equivalent problem of learning multiple policies in MORL. 
We do so by transforming a set of MDPs, as defined in Eq.~\eqref{eq:multi-task-mdp}, into a MOMDP, such that the set of optimal policies for solving all tasks in Eq.~\eqref{eq:multi-task-mdp} is equal to the set of policies that solve the corresponding MOMDP; that is, the policies in the CCS.

Recall that $\vect{\phi}(s,a,s')$ is a $d$-dimensional vector containing the $d$ reward features used to construct SFs. We construct a MOMDP with $m{=}d$ objectives such that its $m$-dimensional reward function, $\vect{R}(s,a,s')$, is defined as  $\vect{\phi}(s,a,s')$. That is, for any $s,a,s'$, we define $R_i(s,a,s') \equiv \phi_i(s,a,s')$, where $R_i$ is the reward function associated with the $i$-th objective of the MOMDP. 
Let $\vect{q}^\pi(s,a)$ be the multi-objective action-value function of the MOMDP. Then,
\begin{align}
     \vect{q}^\pi(s,a) &\equiv \Ex_\pi \left[ \sum_{i=0}^{\infty} \gamma^i \vect{R}_{t+i}  | S_t=s, A_t=a \right]\\
          &= \Ex_\pi \left[ \sum_{i=0}^{\infty} \gamma^i \vect{\phi}_{t+i}  | S_t=s, A_t=a \right]\\
          &\equiv \vect{\psi}^\pi(s,a) \label{eq:equiv_q_psi}.
\end{align}
Therefore, any algorithms capable of learning SFs $\vect{\psi}^\pi(s,a)$ can be used to learn the multi-objective action-value $\vect{q}^\pi(s,a)$ of a corresponding MOMDP, and vice-versa.
Recall that $\vect{\psi}^\pi$ is the expected SF vector associated with an arbitrary policy $\pi$.
Under the previously-introduced definition that $\vect{R}(s,a,s') = \vect{\phi}(s,a,s')$, we can show that the multi-objective policy value is equal to the expected SF vector:
$\vect{v}^\pi \equiv \Ex_{S_0 \sim \mu} \left[ \vect{q}^\pi (S_0, \pi(S_0)) \right] = \Ex_{S_0 \sim \mu} \left[ \vect{\psi}^\pi (S_0, \pi(S_0)) \right] = \vect{\psi}^\pi $, 
where the second equality follows from the identity in Eq.~\eqref{eq:equiv_q_psi}.
Let $\vect{\psi}^\pi$ be the SF vector associated with any $\vect{v}^\pi \in \mathcal{F}$. Then, the original definition of CCS can be rewritten by replacing each occurrence of $\vect{v}^\pi \in \mathcal{F}$ with its corresponding $\vect{\psi}^\pi$:
\begin{align}
    \ccs &{\equiv}\{\vect{v}^\pi {\in} \mathcal{F} \ | \ \exists \vect{w}\ \text{s.t.}\ \forall \vect{v}^{\pi'} {\in} \mathcal{F}, \vect{v}^{\pi} {\cdot} \vect{w} {\geq} \vect{v}^{\pi'} \cdot \vect{w} \}\\
    &{=}\{\vect{\psi}^\pi \ | \ \exists \vect{w}\ \text{s.t.}\ \forall \vect{\psi}^{\pi'}, \vect{\psi}^\pi \cdot \vect{w} \geq \vect{\psi}^{\pi'} \cdot \vect{w} \}\\
    &{=}\{\vect{\psi}^\pi \ | \ \exists \vect{w}\ \text{s.t.}\ \forall \pi', v^{\pi}_{\vect{w}} \geq v_{\vect{w}}^{\pi'} \}. \label{eq:new_ccs}    
\end{align}
Let $\Pi_{\mathrm{CCS}}$ be the set of all policies whose expected SF $\vect{\psi}^\pi$ is in the CCS. All such policies have the property that their value is greater than the value of any other policies, on at least one task, $\vect{w}$. Let $\vect{w}_n$ be \textit{any} new task of interest, and let $\pi^*_n$ be an optimal policy for solving this task. Because $\pi^*_n$ is optimal, it follows that its value, $v^{\pi^*_n}_{\vect{w}}$, is greater than the value of all other policies, $\pi'$, on task $\vect{w}$. That is, $v^{\pi^*_n}_{\vect{w}} \geq v^{\pi'}_{\vect{w}}$. Thus, by the definition of CCS in Eq.~\eqref{eq:new_ccs}, it must be the case that $\vect{\psi}^{\pi^*_n} \in$ CCS and $\pi^*_n \in \Pi_{\mathrm{CCS}}$. This implies that for \textit{any} novel task, whose linear reward function is represented by weights $\vect{w}_n$, a corresponding optimal policy is in $\Pi_{\mathrm{CCS}}$. 

The above realization implies that if one can learn a CCS, it is possible to \textit{directly identify an optimal policy for any linearly-expressible tasks}. Thus, we may reuse algorithms (from the MORL literature) tailored to construct CCS's to derive a set of policies, $\Pi_{\mathrm{CCS}}$, such that, given any MDP with a linear reward function (that is, MDPs $M \in \mathcal{M}^\phi$), its corresponding optimal policy is in $\Pi_{\mathrm{CCS}}$. That is, knowledge of a CCS allows the agent to directly identify optimal solutions to any MDPs with linear reward functions.

\subsection{Theoretical Results}
\label{sec:theoretical-results}

We now prove that learning a CCS in the form of Eq.~\ref{eq:new_ccs} allows for the problem defined in Eq.~\eqref{eq:problem} to be solved.
In particular, we show that by learning a CCS and performing GPI on the corresponding policy set, $\Pi_{\mathrm{CCS}}$, one can guarantee optimal performance of the GPI policy on \textit{all} reward weight vectors $\vect{w} \in \mathcal{W}$. Finally, we also introduce a bound on the performance of the GPI policy when only a partial CCS is available. 
The proofs of all lemmas and theorems in this section can be found in Appendix~\ref{sec:proofs}.

First, we define a weaker strategy for performing policy transfer that is commonly used in MORL settings. Let a \textit{Set Max Policy} (SMP) \citep{Zahavy+2021} be the best policy in some set $\Pi$ for a given reward weight vector $\vect{w}$:
\begin{equation}
    \pismp(s; \vect{w}) = \pi'(s), \text{ where } \pi' = \argmax_{\pi \in \Pi} v^{\pi}_{\vect{w}}.
\end{equation}
Let the value of this policy be $v^{\text{SMP}}_{\vect{w}} = \max_{\pi \in \Pi} v^\pi_{\vect{w}}$.
\citet{Zahavy+2021} showed that for any weight vector $\vect{w} \in \mathcal{W}$ and policy set $\Pi$, it follows that  $v^{\text{GPI}}_{\vect{w}} \geq v^{\text{SMP}}_{\vect{w}}$.

\begin{lemma}
\label{lemma:smp}
Let $\Pi$ be a set of policies and $\vect{w}$ an arbitrary weight vector. If an optimal policy for the reward $r_{\vect{w}}$ is in $\Pi$, then $v^{\text{SMP}}_{\vect{w}} = v^*_{\vect{w}}$.
\end{lemma}
\begin{theorem}
\label{th:gpi-ccs}
Let $\Pi \equiv \{\pi_i\}_{i=1}^{n}$ be a set of policies such that the set of their expected SFs, $\Psi = \{\vect{\psi}^{\pi_i}\}_{i=1}^{n}$, 
constitute a CCS (Eq.~\eqref{eq:new_ccs}).
Then, given any weight vector $\vect{w} {\in} \mathcal{W}$, the GPI policy $\pigpi(s;\vect{w}) \in \argmax_{a \in \mathcal{A}} \max_{\pi \in \Pi} q_{\vect{w}}^{\pi}(s,a)$ is optimal with respect to $\vect{w}$: $v_{\vect{w}}^{\text{GPI}} = v_{\vect{w}}^{*}$.
\end{theorem} 
Theorem~\ref{th:gpi-ccs} shows that learning a CCS in the form of Eq.~\eqref{eq:new_ccs} guarantees optimal behavior when GPI is used to identify a policy for any given task.
However, performing GPI also provides improvement upon a set of policies, $\Pi$, even when an incomplete CCS is available:
\begin{definition}
\label{def:epsilon-ccs}
A SF set $\Psi = \{\vect{\psi}^{\pi_i}\}_{i=1}^{n}$ is an $\epsilon$-CCS if
\begin{align*}
    \forall \vect{w} \in \mathcal{W}, \  
    &\max_{\vect{\psi}^\pi \in \ccs} \vect{\psi}^\pi  \cdot \vect{w} - \max_{\vect{\psi}^\pi \in \Psi} \vect{\psi}^\pi \cdot \vect{w} \leq \epsilon \\
    &\Rightarrow v^*_{\vect{w}} - v^{\text{SMP}}_{\vect{w}} \leq \epsilon .
\end{align*}
\end{definition}
\begin{definition}
Given a SF set $\Psi = \{\vect{\psi}^{\pi_i}\}_{i=1}^{n}$, the GPI-expanded SF set $\Psi^{\text{GPI}}$ is defined as the set
\begin{equation*}
\label{eq:gpi-set}
    \Psi^{\text{GPI}} = \{ \vect{\psi}^\pi | \ \pi \in \{\pi^\text{GPI}(\cdot; \vect{w}) \text{ for all } \vect{w} \in \mathcal{W}\} \}.
\end{equation*}
\end{definition}
\begin{theorem}
\label{th:gpi-epsilon-CCS}
Let $\Pi = \{\pi^*_i\}_{i=1}^{n}$ be a set of optimal policies with respect to weights $\{\vect{w}_i\}_{i=1}^{n}$, such that their SF set $\Psi = \{\vect{\psi}^{\pi^*_i}\}_{i=1}^{n}$ is an $\epsilon_1$-CCS according to Def.~\eqref{def:epsilon-ccs}. Let $\vect{\phi}_{\text{max}} = \max_{s,a} ||\vect{\phi}(s,a)||$. Then, the GPI-expanded SF set $\Psi^{\text{GPI}}$ is an $\epsilon_2$-CCS where:
\begin{equation*}
    \epsilon_2 \leq \min \{\epsilon_1, \frac{2}{1-\gamma} \vect{\phi}_{\text{max}} \max_{\vect{w} \in \mathcal{W}}\min_{i} ||\vect{w} - \vect{w}_i|| \} .\label{eq:opt_gap}
\end{equation*}
\end{theorem}

Intuitively, Theorem~\eqref{th:gpi-epsilon-CCS} implies that if a CCS is incomplete, it is possible to bound the performance gap between the GPI policy (considering an adversarially-chosen task, $\vect{w}$) and the optimal policy for that task. 

\section{Constructing a Set of Policies with Optimistic Linear Support}
\label{sec:method}

In this section, we introduce SFOLS, an SF-based extension of the OLS algorithm \citep{Roijers2016}. It incrementally learns a set of policies, $\Pi$, such that their corresponding successor feature set, $\Psi$, converges to the CCS (Eq.~\eqref{eq:new_ccs}). 
The policies in $\Pi$, at any given iteration of SFOLS, can be combined via GPI to derive an expanded set of solutions $\Psi^{\text{GPI}}$. Because SFOLS monotonically increases the size of $\Pi$, and because it converges to the CCS (as discussed later in this section), it follows from Theorem~\eqref{th:gpi-epsilon-CCS} that the optimality gap in Eq.~\eqref{eq:opt_gap} decreases to zero. In other words, SFOLS is guaranteed to converge to a set of policies, $\Pi$, whose combination via GPI is capable of directly produce optimal solutions to any linearly-expressible tasks.
Pseudocode for SFOLS is shown in Algorithm~\ref{alg:ols}.

\begin{algorithm}[tb]
\caption{SFs Optimistic Linear Support (SFOLS)}
\label{alg:ols}
\begin{algorithmic}[1]
\STATE {\bfseries Initialize:} $\Pi {\leftarrow} \{\}$; $\Psi {\leftarrow} \{\}$; $\mathcal{W}_{exp} {\leftarrow} \{\}$; $Q {\leftarrow} \{\}$; 

\FOR{each extremum of the weight simplex $\vect{w}_e \in \mathcal{W}$}
\STATE {Add $\vect{w}_e$ to $Q$ with maximum priority}
\ENDFOR

\REPEAT
    
\STATE{$\vect{w} \leftarrow \text{pop weight with maximum priority in } Q$}
\STATE{ $\pi, \vect{\psi}^\pi \leftarrow \text{solve task } (\mathcal{S},\mathcal{A},p,r_{\vect{w}},\mu,\gamma)$}
\STATE{Add $\vect{w}$ to $\mathcal{W}_{exp}$}

\IF{$\vect{\psi}^\pi \notin \Psi$}
\STATE{Remove from $Q$ all $\vect{w}'$ s.t. $\vect{\psi}^\pi \cdot \vect{w}' > v^{\text{SMP}}_{\vect{w}'}$}
\STATE{$\mathcal{W}_{c} \leftarrow \mathrm{CornerWeights}(\vect{\psi}^\pi, \vect{w}, \Psi)$}
\STATE{Add $\vect{\psi}^\pi$ to $\Psi$ and $\pi$ to $\Pi$}

\FOR{$\vect{w}' \in \mathcal{W}_{c}$} 
\STATE{$\Delta(\vect{w}') \leftarrow \mathrm{EstimateImprovement}(\vect{w}', \Psi, \mathcal{W}_{exp})$}

\STATE{Add $\vect{w}'$ to $Q$ with priority $\Delta(\vect{w}')$}

\ENDFOR

\ENDIF
\UNTIL{$Q$ \text{is empty}}

\STATE{\bfseries return $\Pi, \Psi$}
\end{algorithmic}
\end{algorithm}

The algorithm starts by inserting into a priority queue, $Q$, the weights in the extrema of the weight simplex $\mathcal{W}$ (i.e., weights in which one component is 1 and all others are 0), assigning them maximum priority.
SFOLS then iteratively pops the weight $\vect{w}$ with the largest priority and uses any RL algorithm to learn a policy, $\pi$, for solving task $\vect{w}$.\footnote{We assume that the agent can \textit{observe} the features $\vect{\phi}_t$ at the current time step $t$. Notice that this is different than assuming prior knowledge of the analytic reward feature function $\vect{\phi}(s,a,s')$. Similar (or more restrictive) assumptions are made in related works \cite{Zahavy+2021,Alver&Precup2021}.} SFOLS also computes the SF, $\vect{\psi}^\pi$, induced by $\pi$.
After computing such a new SF, SFOLS identifies novel weight vectors to add to $Q$ via a procedure that computes \textit{corner weights}, as described below.\footnote{In Appendix~\ref{sec:corner-weights} we detail how corner weights can be computed given any set of SFs (line 11 of Algorithm~\ref{alg:ols}).} SFOLS is guaranteed to stop after a finite number of iterations, when $Q$ is empty. This guarantee follows from similar properties discussed by \citet{Roijers2016}. Thus, as SFOLS identifies more weight vectors (i.e., tasks) to be processed, it incrementally expands the set $\Pi$ until this set converges to a complete CCS.

\begin{figure}[!tb]
\vskip 0.1in
\begin{center}
\centerline{
\includegraphics[width=0.84\columnwidth,align=c]{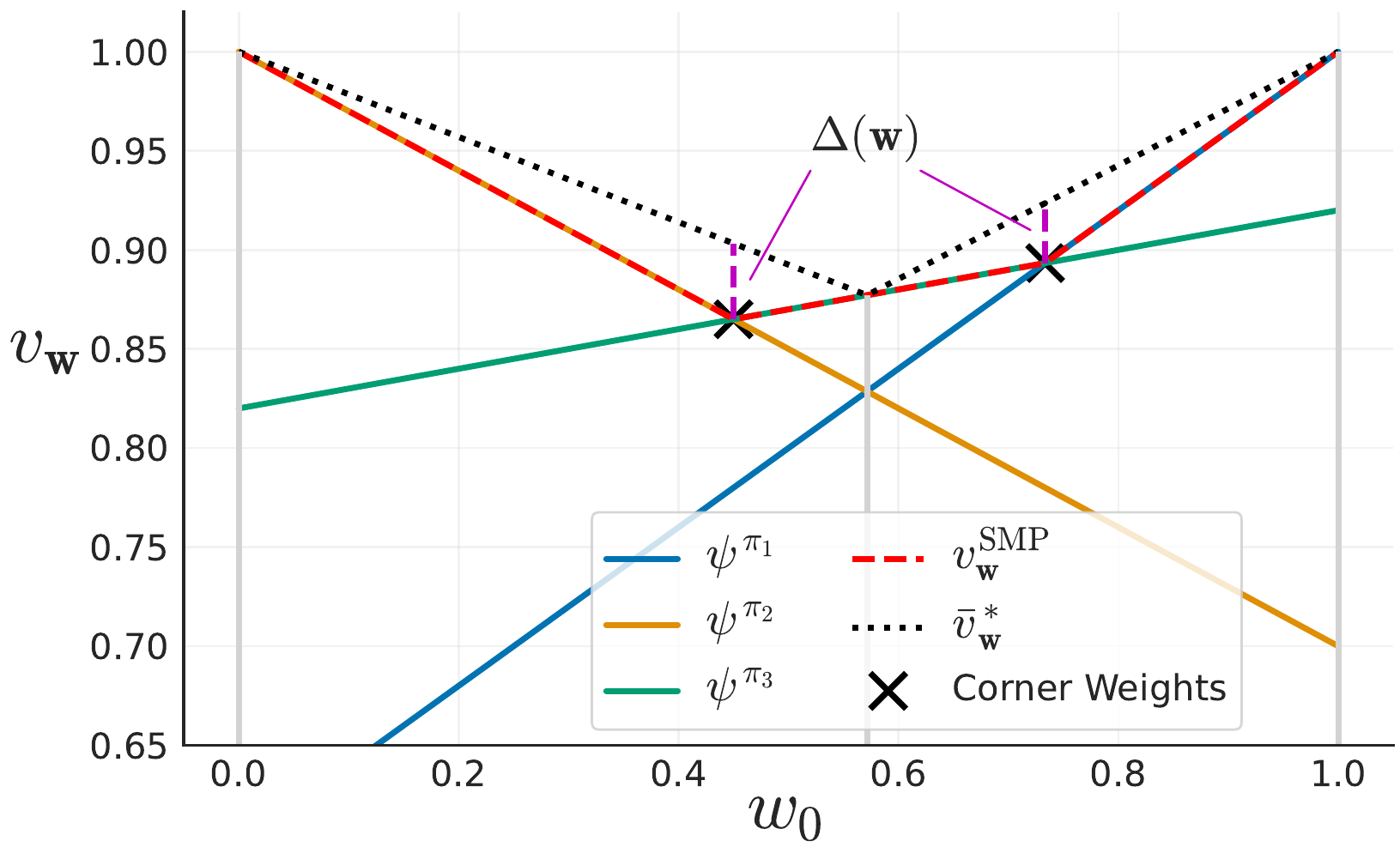}
}
\caption{Example of two corner weights for the SF set $\Psi = \{\vect{\psi}^{\pi_1},\vect{\psi}^{\pi_2},\vect{\psi}^{\pi_3}\}$. Each curve shows how the value of a given policy changes as a function of the task, $\vect{w}$. Here, there are two reward features and thus two weights. We omit $w_1 = 1 - w_0$. Corner weights are the weights in which the value of the SMP policy ($v^{\text{SMP}}_{\vect{w}} = \max_{\pi \in \Pi} \vect{\psi}^\pi \cdot \vect{w}$) changes slope.}
\label{fig:olsexample}
\end{center}
\vskip -0.2in
\end{figure}

To determine which task SFOLS is going to solve at each iteration, it only examines tasks associated with \textit{corner weights}. 
Consider the curve describing the value of the SMP policy, $v^{\text{SMP}}_{\vect{w}} \equiv \max_{\pi \in \Pi} \vect{\psi}^\pi \cdot \vect{w}$, as a function of the task $\vect{w}$. Such a curve forms a \textit{piecewise-linear and convex} (PWLC) surface \citep{Roijers2016}. Corner weights are defined as the points along this surface where it changes slope. These points can be observed in Figure~\ref{fig:olsexample} by analyzing the dashed red curve. Corner weights, here, are denoted by black crosses.
SFOLS can safely consider only corner weights due to the following theorem by \citet{Cheng1988}:
\begin{theorem}
\label{th:cheng}
\cite{Cheng1988} The maximum value of
\begin{align}
    \max_{\vect{w} \in \mathcal{W}, \vect{\psi}^\pi \in \ccs} \min_{\pi' \in \Pi}& \vect{\psi}^\pi \cdot \vect{w} - \vect{\psi}^{\pi'} \cdot \vect{w}\\
    = \max_{\vect{w} \in \mathcal{W}}& \ v^*_{\vect{w}} - v^{\text{SMP}}_{\vect{w}},
\end{align}
is at one of the corner weights of $v^{\text{SMP}}_{\vect{w}} = \max_{\pi \in \Pi} \vect{\psi}^\pi \cdot \vect{w}$.
\end{theorem}
As a result of Theorem~\ref{th:cheng}, corner weights represent tasks whose values under the SMP policy are maximally incorrect with respect to their optimal values. For this reason, they are the best candidate tasks to learn next---intuitively, they are the tasks that the agent knows the least about.
Theorem~\ref{th:cheng} guarantees the correctness of OLS and SFOLS: if at a given iteration all corner weights have a maximal improvement of zero, then the algorithm has identified the complete CCS. Furthermore, Theorem 8 of \cite{Roijers2016} extends such property to the case when the underlying RL algorithm learns $\epsilon$-optimal policies. In this case, both OLS and SFOLS are guaranteed to produce an $\epsilon$-CCS.
Importantly, the number of iterations until SFOLS converges is bounded by $\mathcal{O}(|\mathrm{CCS}| + |\mathcal{W}_{\mathrm{CCS}}|)$, where $|\mathrm{CCS}|$ is the size of the CCS and $|\mathcal{W}_{\mathrm{CCS}}|$ is the number of corner weights.\footnote{Notice, however, that one can also stop the algorithm earlier, when no weights in the priority queue $Q$ have priority greater than a desired optimality threshold $\epsilon$. This results in an $\epsilon$-CCS.}

Theorem~\ref{th:cheng} ensures that the most promising task to practice is located at \textit{one} of the (possibly many) corner weights. A simple heuristic for exploring these tasks is to prioritize them by estimating the difference between $v^{\text{SMP}}_{\vect{w}}$, and an optimistic upper bound on that task's optimal value, $\bar{v}^*_{\vect{w}}$.
This is known as the \textit{optimistic maximal improvement}, $\Delta(\vect{w}) = \bar{v}^*_{\vect{w}} - v^{\text{SMP}}_{\vect{w}}$, where $\bar{v}^*_{\vect{w}}$ is an optimistic upper-bound for $v^*_{\vect{w}}$ (see the black dotted line in Figure~\ref{fig:olsexample}). Notice that this heuristic only changes the order in which corner weights are explored, but the algorithm still explores all of them in the limit. Therefore, this heuristic does not affect the optimality of SFOLS.
The value of $\bar{v}^*_{\vect{w}}$ can be efficiently computed by the linear program in Algorithm~\ref{alg:improvement} with an off-the-shelf solver \citep{Diamond&Boyd2016}.
Interestingly, the upper-bound $\bar{v}^*_{\vect{w}}$ is equivalent to the one introduced by \citet{Nemecek&Parr2021}, which computes the dual version of the linear program. In Appendix~\ref{sec:equivalence-nemecek} we prove this equivalence.

\begin{algorithm}[tb]
\caption{Estimate Improvement}
\label{alg:improvement}

\begin{algorithmic}[1]
   \STATE {\bfseries Input:} New weight vector $\vect{w}$, SFs set $\Psi$, set of weights, $\mathcal{W}_{exp}$, for which optimal policies are already known.
   
   \STATE{Let $\bar{v}^*_{\vect{w}}$ be the optimistic upper bound on $v^*_{\vect{w}}$, computed by the following linear program:}
   \vspace{-0.3cm}
   \begin{align*}
    \max \vect{\psi} &\cdot \vect{w}\\
    \text{subject to } \vect{\psi} &\cdot \vect{w}' \leq v^{\text{SMP}}_{\vect{w}'}, \text{ for all } \vect{w}' \in \mathcal{W}_{exp}
   \end{align*}
    \vspace{-0.6cm}
   \STATE{$\Delta(\vect{w}) \leftarrow \bar{v}^*_{\vect{w}} - v^{\text{SMP}}_{\vect{w}}$}
   
\STATE {\bfseries return} $\Delta(\vect{w})$
   
\end{algorithmic}
\end{algorithm}

\section{Experiments}

We compare SFOLS with the Worst Case Policy Iteration (WCPI) algorithm \citep{Zahavy+2021} and other baselines. 
WCPI works by iteratively learning a new policy that is optimal for the reward function under which the previously-learned policies perform the worst.
This worst-case reward function is defined as $\bar{\vect{w}} = \argmin_{\vect{w}\in\mathcal{W}}\max_{\pi\in\Pi} \vect{\psi}^\pi \cdot \vect{w}$.
WCPI stops when the value of $v^{\text{SMP}}_{\bar{\vect{w}}}$ no longer improves between successive iterations.
The set of policies learned with WCPI is provably optimal with respect to the worst-case reward of the MDP. Empirically, it has been shown to produce a diverse set of policies with good performance over randomly selected test tasks. Additionally, we compare with a baseline that, at each iteration, learns a policy to solve a randomly-selected task. This is also the approach used in \cite{Nemecek&Parr2021}.
We perform these comparisons in three scenarios: a classic MORL environment and two well-known benchmark environments used in the SFs literature. Additional details can be found in Appendix~\ref{ap:experiment}.

\begin{figure*}[tb]
\vskip 0.1in
\begin{center}
\centerline{
\includegraphics[width=0.25\linewidth,align=c]{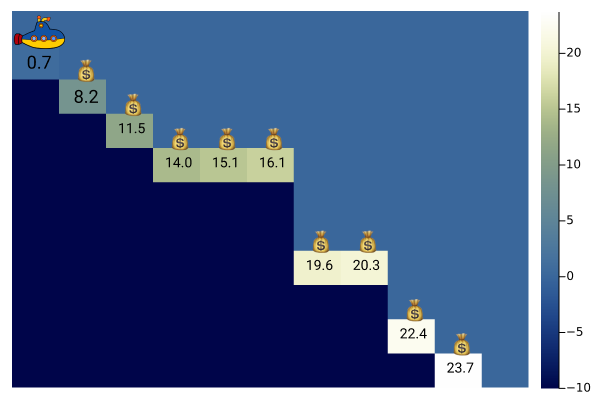}
\includegraphics[width=0.35\linewidth,align=c]{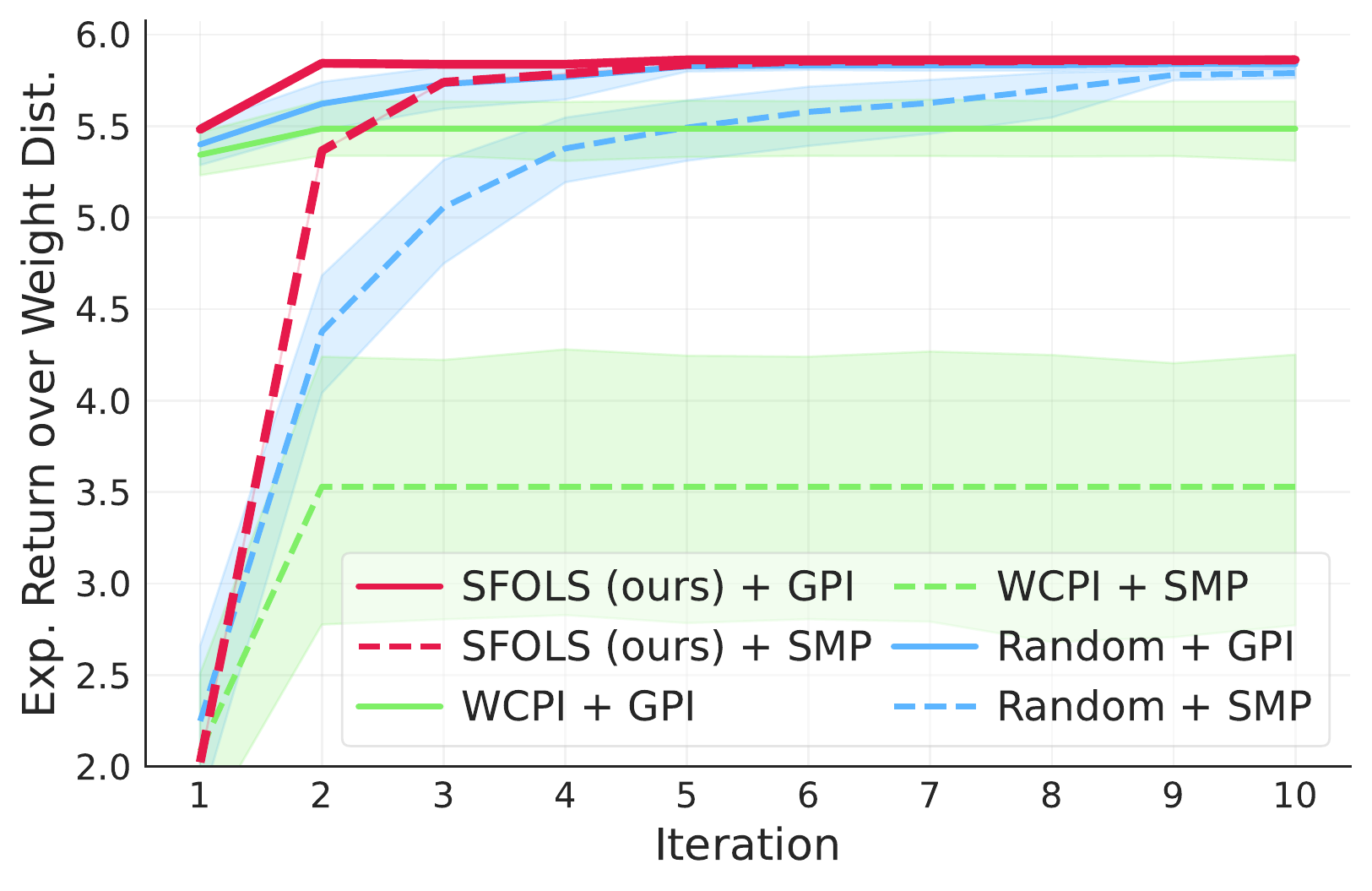}
\includegraphics[width=0.35\linewidth,align=c]{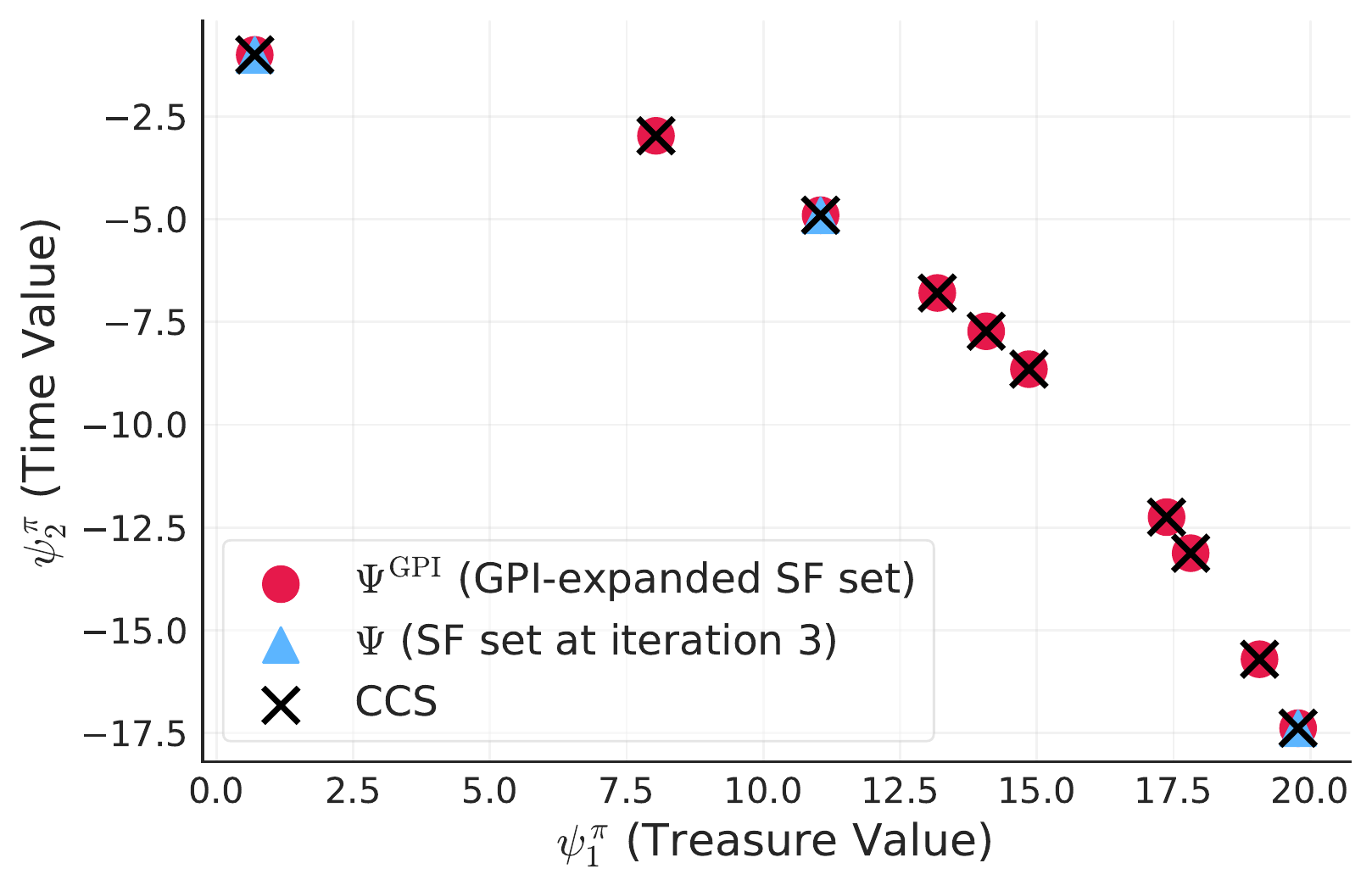}
}
\caption{\textbf{Left:} DST domain. \textbf{Middle:} Expected return of each algorithm over the task/reward weight distribution, $\mathcal{W}$, when evaluated using either GPI or SMP. \textbf{Right:} SFOLS recovers the complete CCS (black crosses) by performing GPI over only \textit{three} policies (blue triangles), thereby identifying all other policies in the CSS (red circles) after only three iterations.}
\label{fig:dst}
\end{center}
\vskip -0.2in
\end{figure*}

\paragraph{Deep Sea Treasure (DST).} We first compare all methods in the DST environment, a classic MORL domain \citep{Abels+2019,Yang+2019}. Here, the agent is a submarine in a $10\times11$-grid (left panel of Figure~\ref{fig:dst}) that must collect a treasure under a time penalty. The first component of the reward feature vector, $\vect{\phi}(s,a,s') \in \mathbb{R}^2$, is the treasure value (or zero when in a blank cell), and the second component is always $-1$.
In Figure~\ref{fig:dst}'s middle panel, we show the expected return (value) achieved by each method when evaluated over a test set of 64 tasks uniformly sampled from $\mathcal{W}$. We report the mean value and its 95\% confidence interval over 30 random seeds. 
First, notice that the WCPI algorithm converges to a sub-optimal policy set, as it does not learn any new policies after it solves the task with the lowest optimal value.
SFOLS, by contrast, and the random baseline, continuously improve their expected performances by learning new policies.

The rightmost plot of Figure~\ref{fig:dst} shows how SFOLS is capable of rapidly identifying the complete CCS, whose policies are shown as black crosses. Notice that after three iterations, SFOLS's SF set $\Psi$ contains only three policies (depicted as blue triangles). Even though SFOLS is forced, here, to operated over a reduced amount of experiences/policies, it is already capable of identifying and recovering \textit{all} policies in the CSS. In particular, the policies that SFOLS identifies by performing GPI over the three policies in $\Psi$ are shown as red circles. Notice, then, that our approach succeeds in its main goal of efficiently selecting tasks to practice in a way that (by combining their corresponding policies via GPI) allows the agent to rapidly reconstruct the entire CCS. 
This emphasizes the potential of using GPI as a method to avoid the costs of explicitly learning a complete CCS---which often contains a large number of policies.
Plots depicting SFOLS' performance after different number of iterations can be found in Appendix~\ref{sec:dst-iterations}.

\begin{figure*}[t!]
\vskip 0.1in
\begin{center}
\centerline{
\includegraphics[width=0.2\linewidth,align=c]{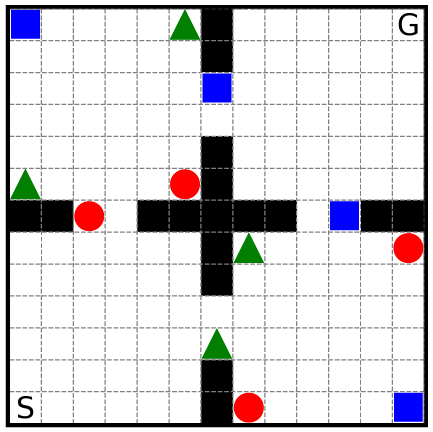}
\includegraphics[width=0.35\linewidth,align=c]{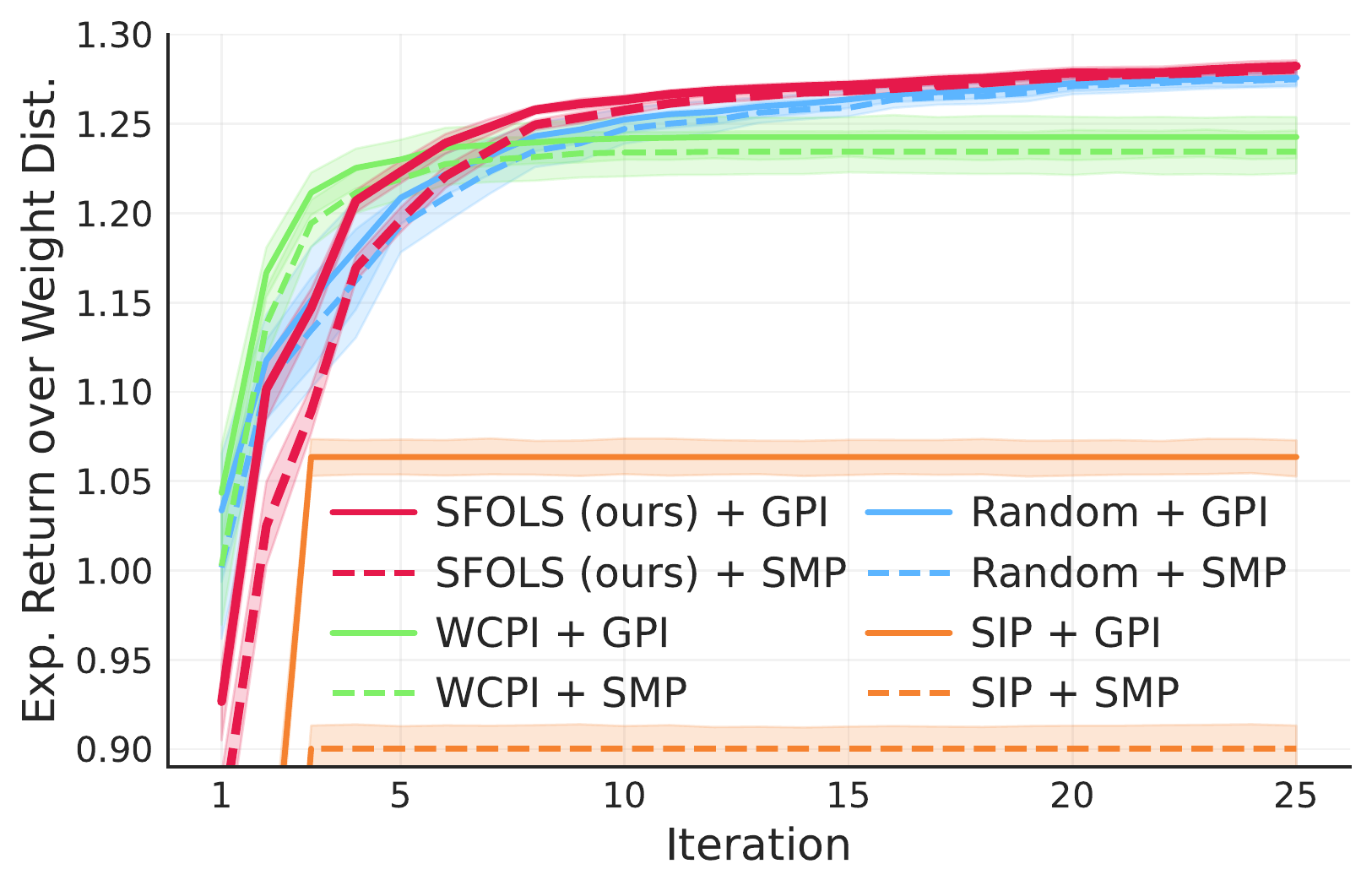}
\includegraphics[width=0.35\linewidth,align=c]{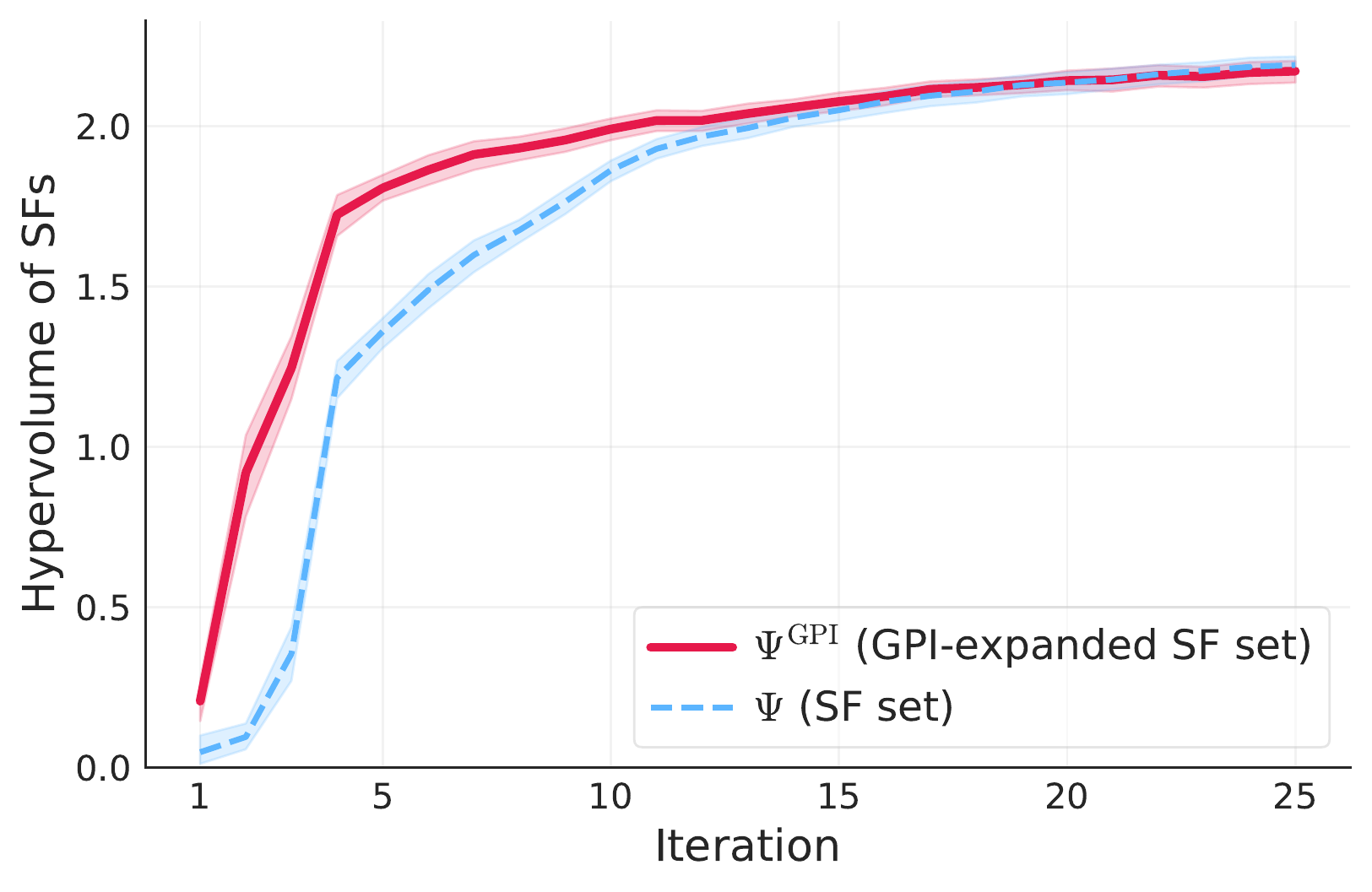}
}
\caption{\textbf{Left:} Four Room domain. \textbf{Middle:} Expected return of each algorithm over the task/reward weight distribution, $\mathcal{W}$, when evaluated using either GPI or SMP.  \textbf{Right:} Volume under the CCS frontier (\textit{hypervolume}) discovered by SFOLS as a function of  iterations. The red curve indicates that SFOLS' hypervolume grows rapidly and that it quickly converges to an almost-complete CCS.}
\label{fig:fourroom}
\end{center}
\vskip -0.2in
\end{figure*}

\paragraph{Four Room.} Next, we evaluate SFOLS in the Four Room domain \cite{Barreto+2017,Gimelfarb+2021}. The Four Room domain has a significantly larger state space than the DST domain.
In this task (depicted in the leftmost panel of Figure~\ref{fig:fourroom}), the reward features are one-hot encoded vectors $\vect{\phi}(s,a,s') \in \{0,1\}^3$ indicating the presence of one of three different classes of objects.
Since this domain satisfies the \textit{independent features} assumption, as defined by \citet{Alver&Precup2021}, we also compare SFOLS with their method: \textit{set of independent policies} (SIP) \cite{Alver&Precup2021}. 
The SIP algorithm learns $d$ policies, where each policy is optimal with respect to one task defined by a weight vector in which only one of its $d$ components is a positive value. All others are negative values.

We can observe (in the middle panel of Figure~\ref{fig:fourroom}) that WCPI constructs a better policy set than the competing methods in the first five iterations. However, after five iterations, it converges to a sub-optimal policy set. 
The SIP algorithm failed to present good performance when evaluated in the test tasks. We believe this occurs because SIP was designed to maximize the undiscounted total reward, and its performance guarantees do not extend to the discounted return setting.
SFOLS, by contrast, keeps improving its expected return over the task distribution after every iteration.
Finally, in the rightmost panel of Figure~\ref{fig:fourroom} we show the volume under the CCS frontier discovered by SFOLS (its \textit{hypervolume}) as a function of the number of iterations. The hypervolume is a widely-used metric deployed to measure the \text{coverage} of the set of solutions over the objectives space \cite{Hayes+2022}. In particular, notice that the red curve indicates that the volume of the CCS frontier identified by SFOLS grows rapidly, thus indicating that after only a few iterations, our method is capable of quickly converging to an almost-complete CCS.
This emphasizes how GPI, when performed over the solutions identified by SFOLS, can efficiently recover novel policies in settings where a computational budget may restrict the agent's capability of learning many policies.

\begin{figure*}[t!]
\vskip 0.1in
\begin{center}
\centerline{
\includegraphics[width=0.2\linewidth,align=c]{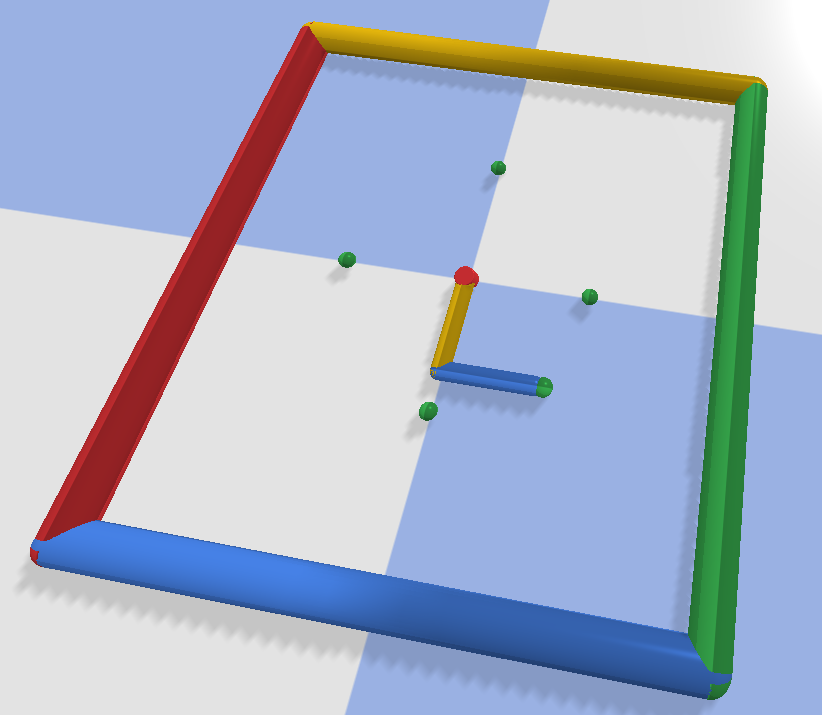}
\includegraphics[width=0.35\linewidth,align=c]{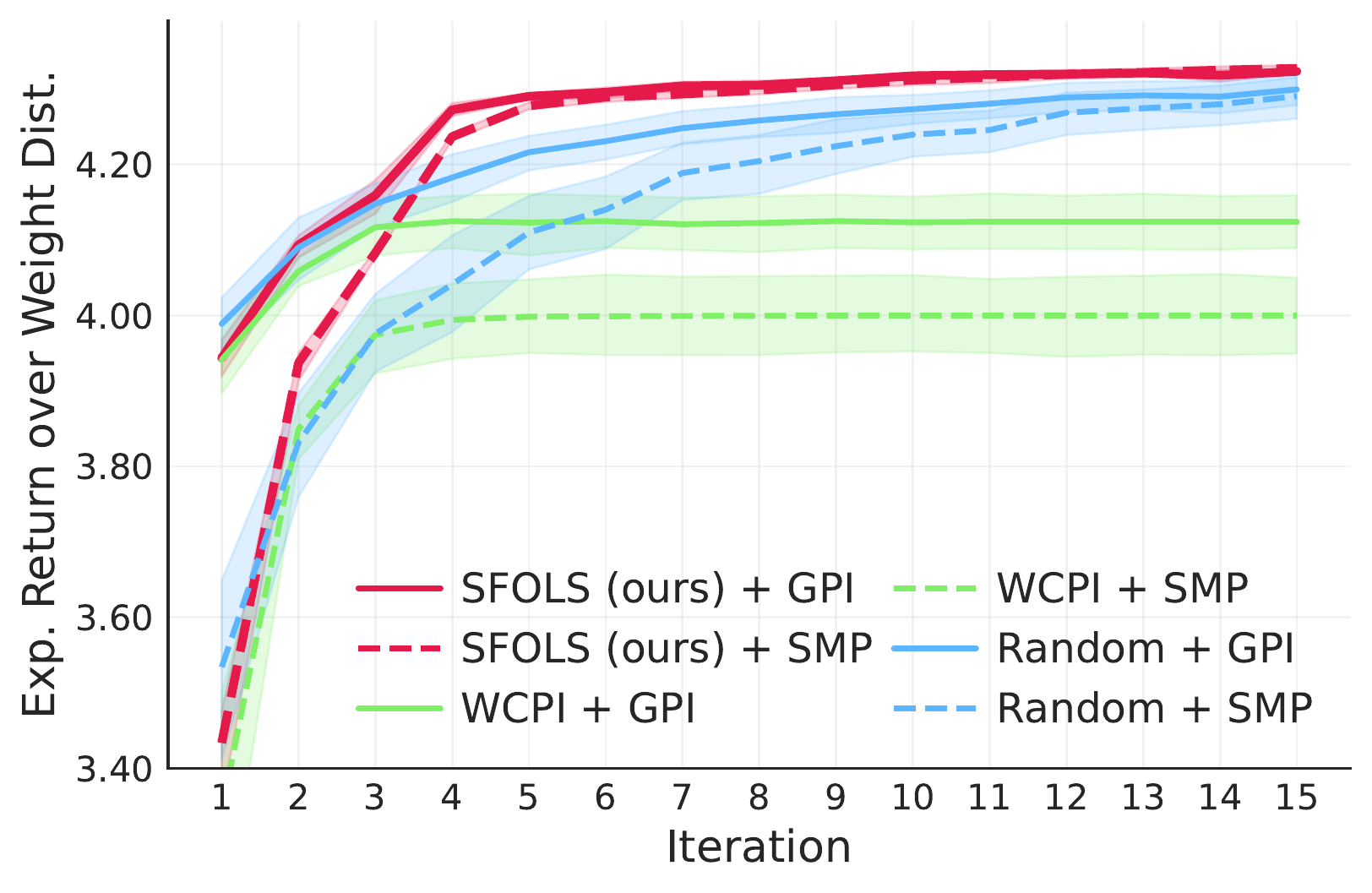}
\includegraphics[width=0.35\linewidth,align=c]{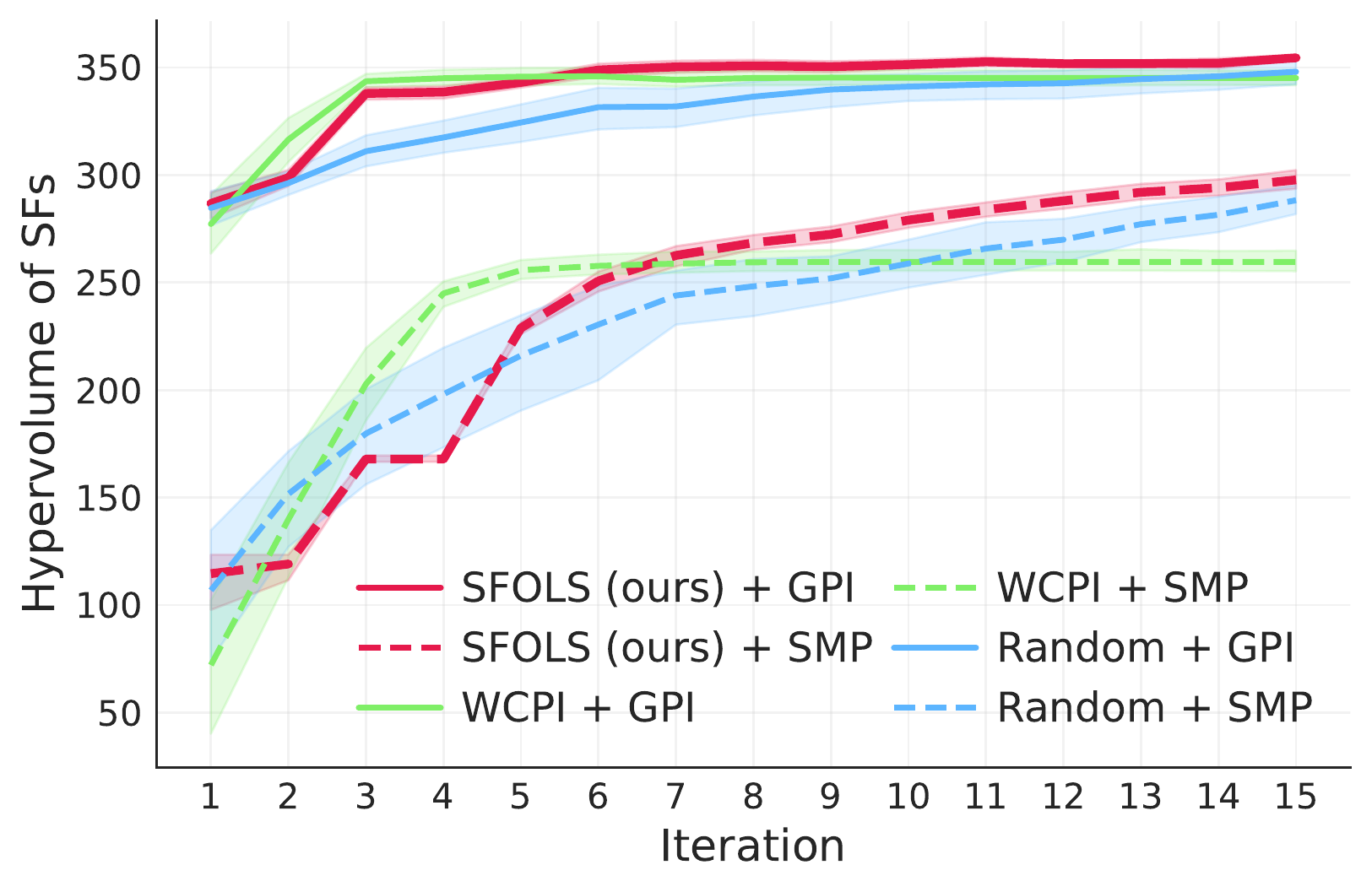}}
\caption{\textbf{Left:} Reacher environment. \textbf{Middle:} Expected return of each algorithm over the task/reward weight distribution, $\mathcal{W}$, when evaluated using either GPI or SMP.  \textbf{Right:} Hypervolume identified by each method (SFOLS, WCPI, Random) at each iteration.}
\label{fig:reacher}
\end{center}
\vskip -0.2in
\end{figure*}

\paragraph{Reacher.} Lastly, we evaluate all algorithms in a setting with a continuous state space that requires using function approximation techniques.
We modify the Reacher environment from PyBullet \cite{Ellenberger2018}, similarly as done in \cite{Barreto+2017,Gimelfarb+2021,Nemecek&Parr2021}.
In this domain, the agent is a robotic arm composed of two segments and that can apply torque to each of its two joints. The reward features $\vect{\phi}(s,a,s') \in \mathbb{R}^4$ are defined as $1$ minus the Euclidean distance from the tip of the arm to four different targets (Figure~\ref{fig:reacher}).
In this domain, we use neural networks to learn the policies' successor features.

The results in the middle and rightmost panels of  Figure~\ref{fig:reacher} show that SFOLS solves the problem after five iterations, while competing methods only approximate (but never reach) the performance of solution after three times more iterations.
In this domain, the GPI policy significantly outperforms the SMP policy up until the fifth iteration of SFOLS. WCPI converges to a sub-optimal policy set (with respect to the test weights), while the random baseline requires more iterations than SFOLS to produce a good policy set.
The hypervolume metric (shown in the rightmost panel of Figure~\ref{fig:reacher}) reveals how GPI generates novel solutions that more efficiently cover the space of solutions.
Finally, notice that WCPI has a higher hypervolume during the first six iterations, which shows its capability to quickly create a diverse behavior basis. 
SFOLS, by contrast, learns policies for the weights/tasks where the optimistic maximal improvement is higher, allowing it to continuously improve its performance over the task distribution.

\begin{figure}[tb]
\vskip 0.0in
\begin{center}
\centerline{
\includegraphics[width=0.82\columnwidth,align=c]{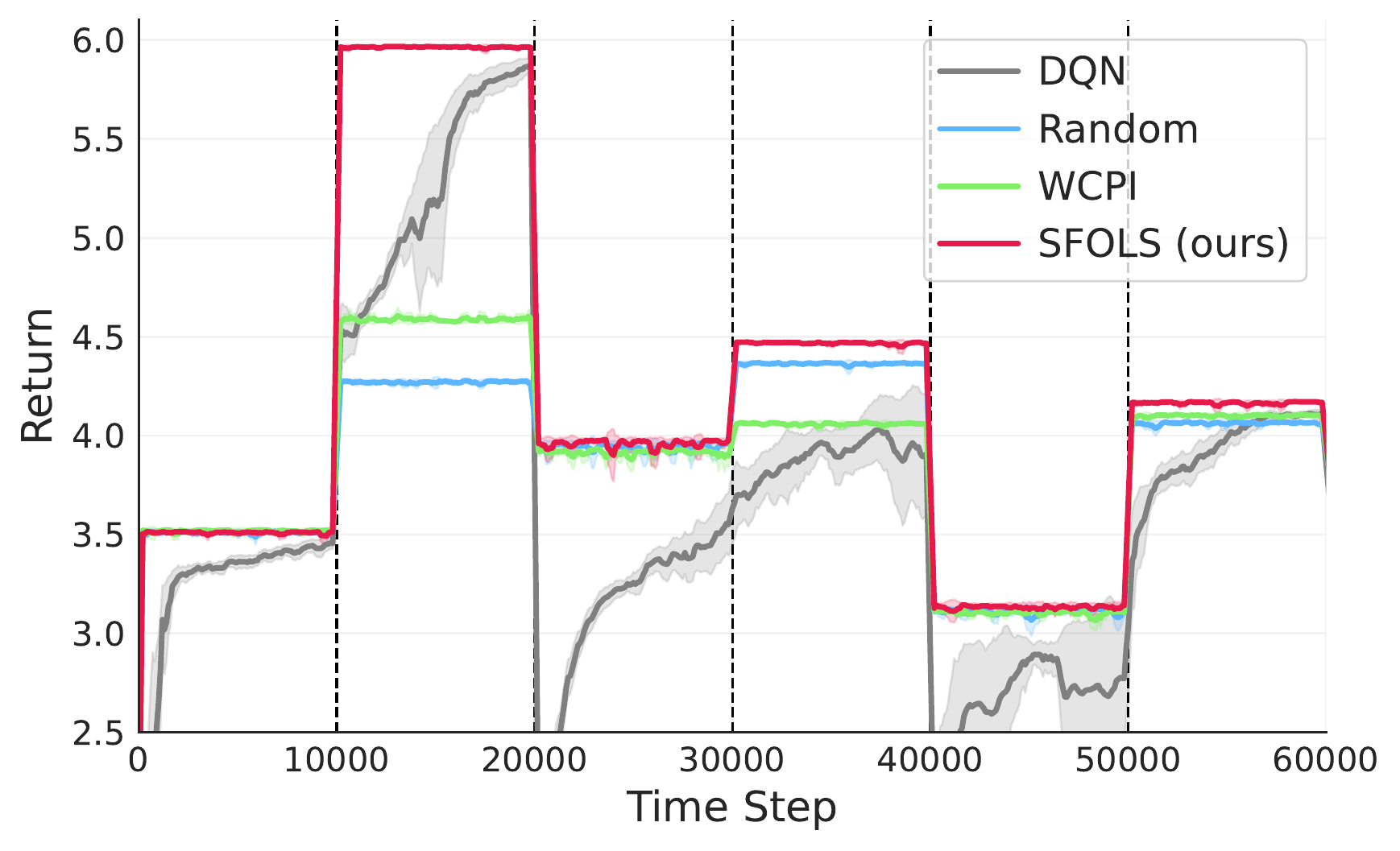}
}
\caption{Episodic return of each algorithm in a lifelong setting in which the reward vector $\vect{w}$ changes every 10,000 time steps.}
\label{fig:lifelong}
\end{center}
\vskip -0.3in
\end{figure}

\paragraph{Lifelong RL.} Finally, we consider a lifelong setting (similar to the one described in \citet{Alver&Precup2021}) in which the task being solved by the agent changes with time in an unpredictable manner. In particular, after every 10,000 steps, a new (previously unknown) task $\vect{w}$ is uniformly sampled from the task space $\mathcal{W}$. The goal of this experiment is to evaluate the \textit{zero-shot performance} of the GPI solution/behavior resulting from the policy sets constructed by each algorithm, after a pre-training period corresponding to the learning process described previously.\footnote{That is, after executing each algorithm for 15 iterations in the setting corresponding to the previous experiment; see Figure~\ref{fig:reacher}.} 
Figure~\ref{fig:lifelong} shows the zero-shot performance of SFOLS and competing methods, including a baseline a DQN agent \cite{Mnih+2015} which has to continuously re-adapt its action-value function. 
In this zero-shot learning setting, algorithms that can adapt more rapidly to novel tasks have better performance (higher return).
As can be seen, the GPI solution/behavior obtained by combining the policies learned by SFOLS \textit{always} results in higher or similar returns than the competing algorithms. This is because the policies associated with the CCS constructed by SFOLS better cover the space of tasks, $\mathcal{W}$.
As expected, the DQN baseline struggles to adapt to changes to the task being tackled by the agent.

\vspace{-0.1cm}
\section{Discussion and Related Work}

To the best of our knowledge, the survey of \citet{Hayes+2022} is the only work that briefly discusses similarities between MORL and SFs. In this paper, we first introduced theoretical results capable of successfully (and formally) combining ideas from both frameworks to derive a novel method that outperforms the state-of-the-art.

The work of \citet{Zahavy+2021} is the closest to ours. 
They proposed the WCPI algorithm to learn a diverse set of policies. 
While this algorithm produces a set of policies guaranteed to be optimal with respect to the worst-case reward, there are no guarantees regarding its performance in relation to the entire space of linear rewards.
We, on the other hand, show theoretically and empirically that SFOLS \textit{can} construct a set of policies optimal to any linear reward.
\citet{Zahavy+2021arxiv} studied the problem of learning a set of policies that maximizes a diversity metric while maintaining good performance in a task given by a fixed extrinsic reward.
We, by contrast, consider the problem of constructing a policy set that allows for optimal (or near-optimal) behavior in any linearly-expressed task.
More recently, \citet{Alver&Precup2021} proposed to construct a set of independent policies as a basis to be used with generalized policy updates.
Their method is guaranteed to be optimal only in the undiscounted case.
Additionally, it assumes that the MDP's transition function is deterministic, while SFOLS does not have this restriction.
\citet{Nemecek&Parr2021} introduced a method that decides whether a policy should be added to the agent's policy library based on an upper-bound computed using the policies' SFs. 
However, they do not provide a way to decide which tasks the agent should learn at each time, and instead train on randomly sampled ones.

Although prior policy transfer methods, which were not based on SFs, have been proposed \citep{Pickett&Barto2002,Fernandez&Veloso2006,Taylor&Stone2007,Abel+2018}, they do not address the problem of constructing a set of policies that allows for optimal behaviors (over arbitrary tasks) to be identified.
An exception is the work of \citet{Tasse+2020,Tasse+2022}, which can recover optimal policies for tasks expressible under their proposed Boolean Task Algebra framework.
Notice, however, that unlike SFOLS their optimality guarantees only hold for goal-based tasks. 

The problem we address in this paper is also related to the one studied in the unsupervised skill discovery literature \cite{Gregor+2016,Eysenbach+2019,Hansen+2020,Liu&Abbeel2021}. Typically, methods studied in this area use information-theoretic objectives to identify (in an unsupervised manner) policy sets  that can be used to enable faster transfer to novel tasks.
Recently, \citet{Eysenbach+2022} studied this problem under a geometric perspective in which policies are defined by their discounted state occupancy. In this case, each policy is associated with an  $|\mathcal{S}|$-dimensional point on the probability simplex. \citet{Eysenbach+2022} define the open problem of skill discovery as the problem of identifying the smallest set of policies such that every vertex of the discounted state occupancy polytope contains at least one policy.
In the discrete state case, in which reward features $\vect{\phi}_t$ are a one-hot representation of the agent's state in $\mathbb{R}^{|\mathcal{S}|}$, SFs correspond to the \textit{successor representation} (SR) \cite{Dayan1993}.
Notably, successor representations are also discounted state occupancy measures. In Appendix~\ref{sec:skill-discovery} we show that SFOLS effectively constructs a CCS over successor representations and thus solves the open problem of skill discovery proposed by \citet{Eysenbach+2022}.

\section{Conclusions}

We showed that any transfer learning problem within the SF framework can be mapped into an equivalent problem of learning multiple policies in MORL under linear preferences. We then introduced a novel SF-based extension of the OLS algorithm (SFOLS) to iteratively construct a set of policies whose SFs form a CCS. Additionally, \textit{we showed that these policies can be combined via GPI to directly identify optimal solutions for any novel linearly-expressible tasks.} To the best of our knowledge, ours is the only method capable of formally guaranteeing such a capability. We empirically showed that SFOLS outperforms state-of-the-art competing algorithms both in classic MORL and SF domains. We believe that our theoretical and empirical findings are relevant to both the MORL and transfer learning/SFs communities.
As future work, we would like to combine SFOLS with \textit{universal successor features approximators} (USFAs) \cite{Borsa+2019}. USFAs may allow us to learn a single model that generalizes over multiple tasks and that better scales to high-dimensional problems.

\section*{Acknowledgements}
We thank Diederik Roijers for insightful discussions, and the anonymous reviewers for their valuable feedback.
This study was financed in part by the following Brazilian agencies: Coordenação de Aperfeiçoamento de Pessoal de Nível Superior - Brazil (CAPES) - Finance Code 001; CNPq (grants 140500/2021-9 and 304932/2021-3); and FAPESP/MCTI/CGI (grant number 2020/05165-1).

\bibliographystyle{icml2022}

\newpage
\appendix
\onecolumn

\section{Proofs}
\label{sec:proofs}

\subsection{Proof of Lemma \ref{lemma:smp}}
\begin{lemma*}
Let $\Pi$ be a set of policies and $\vect{w}$ an arbitrary weight vector. If an optimal policy for the reward $r_{\vect{w}}$ is in $\Pi$, then $v^{\text{SMP}}_{\vect{w}} = v^*_{\vect{w}}$.
\end{lemma*}

\begin{proof}
\begin{align*}
    v^{\text{SMP}}_{\vect{w}} &= \max_{\pi \in \Pi} \vect{\psi}^\pi \cdot \vect{w} \\
    &= \vect{\psi}^{\pi^*_{\vect{w}}} \cdot \vect{w} \\
    &= v^*_{\vect{w}} \quad \text{(Definition of Optimal Value)}.
\end{align*}
\end{proof}

\subsection{Proof of Theorem~\ref{th:gpi-ccs}}
\begin{theorem*}
Let $\Pi \equiv \{\pi_i\}_{i=1}^{n}$ be a set of policies such that the set of their expected SFs, $\Psi = \{\vect{\psi}^{\pi_i}\}_{i=1}^{n}$, 
constitute a CCS (Eq.~\eqref{eq:new_ccs}).
Then, given any weight vector $\vect{w} {\in} \mathcal{W}$, the GPI policy $\pigpi(s;\vect{w}) \in \argmax_{a \in \mathcal{A}} \max_{\pi \in \Pi} q_{\vect{w}}^{\pi}(s,a)$ is optimal with respect to $\vect{w}$: $v_{\vect{w}}^{\text{GPI}} = v_{\vect{w}}^{*}$.
\end{theorem*}

\begin{proof}\footnote{This proof follows part of the proof of Lemma 2 of \citet{Zahavy+2021}.}
\begin{align*}
    v^{\text{GPI}}_{\vect{w}}(s) &= q^{\text{GPI}}_{\vect{w}}(s,\pigpi(s)) \quad \forall s \ \text{}\ \\
    &\geq \max_{\pi \in \Pi, a \in \mathcal{A}} q^{\pi}_{\vect{w}}(s,a)  \quad \forall s \ \text{(GPI Theorem)} \\
    &\geq \max_{\pi \in \Pi} v^{\pi}_{\vect{w}}(s) \quad \forall s \ \text{.}
\end{align*}
Taking the expected value with respect to the initial state distribution, $\mu$, on both sides, gives us:
\begin{align*}
    \Ex_{S_0 \sim \mu} \left[v^{\text{GPI}}_{\vect{w}}(S_0)\right] &\geq  \Ex_{S_0 \sim \mu} \left[ \max_{\pi \in \Pi} v^{\pi}_{\vect{w}}(S_0) \right]\ \text{} \\
     &\geq \max_{\pi \in \Pi} \Ex_{S_0 \sim \mu} \left[ v^{\pi}_{\vect{w}}(S_0) \right], \quad \text{} \\
    v^{\text{GPI}}_{\vect{w}} &\geq \max_{\pi \in \Pi} v^{\pi}_{\vect{w}} \ \text{} \\
    &= v^{\text{SMP}}_{\vect{w}} \quad \text{(SMP Definition)}\\
    &= v^*_{\vect{w}} \quad \text{(Lemma 1 and CCS Definition).}
\end{align*}

The last equality comes from the fact that given any weight vector $\vect{w} \in \mathcal{W}$, there exists an optimal policy $\pi^*_{\vect{w}}$ such that $\vect{\psi}^{\pi^*_{\vect{w}}} \in \ccs$.
\end{proof}

\subsection{Proof of Theorem~\ref{th:gpi-epsilon-CCS}}
\begin{theorem*}
Let $\Pi = \{\pi^*_i\}_{i=1}^{n}$ be a set of optimal policies with respect to weights $\{\vect{w}_i\}_{i=1}^{n}$, such that their SF set $\Psi = \{\vect{\psi}^{\pi^*_i}\}_{i=1}^{n}$ is an $\epsilon_1$-CCS according to Def.~\eqref{def:epsilon-ccs}. Let $\vect{\phi}_{\text{max}} = \max_{s,a} ||\vect{\phi}(s,a)||$. Then, the GPI-expanded SF set $\Psi^{\text{GPI}}$ is an $\epsilon_2$-CCS where:
\begin{equation*}
    \epsilon_2 \leq \min \{\epsilon_1, \frac{2}{1-\gamma} \vect{\phi}_{\text{max}} \max_{\vect{w} \in \mathcal{W}}\min_{i} ||\vect{w} - \vect{w}_i|| \} .
\end{equation*}
\end{theorem*}

\begin{proof}
We start proving that $\epsilon_2 \leq \epsilon_1$. From Lemma 2 of \citet{Zahavy+2021}, we have that:
\begin{align*}
    \forall{\vect{w}} \in \mathcal{W}, \ &v^{\text{GPI}}_{\vect{w}} \geq v^{\text{SMP}}_{\vect{w}}, \\
    &v^*_{\vect{w}} - v^{\text{GPI}}_{\vect{w}} \leq v^*_{\vect{w}} - v^{\text{SMP}}_{\vect{w}} .
\end{align*}
Because $\Psi$ is an $\epsilon_1$-CCS, then: 
\begin{align*}
    \forall{\vect{w}} \in \mathcal{W}, \
    &v^*_{\vect{w}} - v^{\text{GPI}}_{\vect{w}} \leq v^*_{\vect{w}} - v^{\text{SMP}}_{\vect{w}} \leq \epsilon_1 .
\end{align*}
Therefore, it must exist an $\epsilon_2 \leq \epsilon_1$ such that:
\begin{align}
\label{eq:upb1}
    \forall{\vect{w}} \in \mathcal{W}, \
    &v^*_{\vect{w}} - v^{\text{GPI}}_{\vect{w}} \leq \epsilon_2 \leq \epsilon_1 .
\end{align}
From Theorem 2 of \citet{Barreto+2017}, we have that $\forall{\vect{w}}\in\mathcal{W}$ and $\forall{(s,a)}\in\mathcal{S}\times\mathcal{A}$:
\begin{align}
\label{eq:barreto-bound}
    q^*_{\vect{w}}(s,a) - q^{\text{GPI}}_{\vect{w}}(s,a) \leq \frac{2}{1-\gamma} \vect{\phi}_{\text{max}} \min_{i} ||\vect{w} - \vect{w}_i|| .
\end{align}
Without loss of generality, consider a new state space $\bar{\mathcal{S}} = \mathcal{S} \cup \{\bar{s}\}$, where $\bar{s}$ is a new dummy initial state in which only a single action $\bar{a}$ is available.
Let $p(s_0|\bar{s},\bar{a}) = d_0(s)$ for all $s_0 \in \mu$, where $d_0(s_0)$ is the original probability of the initial state being $s_0$. 
Notice that this does not change the values of any policy for the states $s \in \mathcal{S}$. Hence:
\begin{align*}
    q^*_{\vect{w}}(\bar{s},\bar{a}) - q^{\text{GPI}}_{\vect{w}}(\bar{s},\bar{a}) &= q^*_{\vect{w}}(\bar{s},\pi^*(\bar{s})) - q^{\text{GPI}}_{\vect{w}}(\bar{s},\pigpi(\bar{s})) ,\\
    &= v^*_{\vect{w}} - v^{\text{GPI}}_{\vect{w}} .
\end{align*}
Because (\ref{eq:barreto-bound}) holds for any $\forall{(s,a)}\in\mathcal{S}\times\mathcal{A}$, it holds for $(\bar{s},\bar{a})$, and we have that:
\begin{align*}
    \forall{\vect{w}}\in\mathcal{W}, \ v^*_{\vect{w}} - v^{\text{GPI}}_{\vect{w}} &\leq \frac{2}{1-\gamma} \vect{\phi}_{\text{max}} \min_{i} ||\vect{w} - \vect{w}_i||.
\end{align*}
Finally, in the worst case:
\begin{align}
\label{eq:upb2}
    \forall{\vect{w}} \in \mathcal{W}, \
    v^*_{\vect{w}} - v^{\text{GPI}}_{\vect{w}} \leq \epsilon_2 \leq \frac{2}{1-\gamma} \vect{\phi}_{\text{max}} \max_{\vect{w} \in \mathcal{W}}\min_{i} ||\vect{w} - \vect{w}_i|| .
\end{align}
Combining upper-bounds (\ref{eq:upb1}) and (\ref{eq:upb2}) for $\epsilon_2$ completes the proof.
\end{proof}

\subsection{Equivalence of \citet{Nemecek&Parr2021} and SFOLS Upper-Bounds}
\label{sec:equivalence-nemecek}

Given a set of policies $\Pi = \{\pi_i\}_{i=1}^{n}$ which are optimal, respectively, with respect to the tasks defined by weight vectors $\{\vect{w}'_i\}_{i=1}^{n}$, and given a novel weight vector $\vect{w} \in \mathbb{R}^d$, \citet{Nemecek&Parr2021} proposed a method to compute an upper-bound for the optimal value $v^*_{\vect{w}}$ by solving the following linear program:
\begin{align*}
    \min& \sum_{i=1}^{n} \alpha_i v^{\pi_i}_{\vect{w}'_i}\\
    \text{subject to}& \sum_{i=1}^{n} \alpha_i w'_{i,j} = w_j, \ j = 1,...,d\\
    &\alpha_i \geq 0,  \ i = 1,...,n .
\end{align*}
where $\vect{\alpha} \in \mathbb{R}^n$ is the vector of variables.
Let $\vect{\psi} \in \mathbb{R}^{m}$ be the dual variables of this linear program. Then, the dual linear program can be expressed as:
\begin{align*}
    \max& \sum_{j=1}^{d} \psi_j w_{j}\\
    \text{subject to}& \sum_{j=1}^{d} \psi_j w'_{i,j} \leq v^{\pi_i}_{\vect{w}'_i}, \ i = 1,...,n\\
    &\psi_j \in \mathbb{R},  \ j = 1,...,d .
\end{align*}
This is the same linear program used by SFOLS in Algorithm~\ref{alg:improvement}, which is adapted from the OLS algorithm \cite{Roijers2016}, to compute the optimistic upper-bound $\bar{v}^*_{\vect{w}}$.
Hence, as the linear programs of \citet{Nemecek&Parr2021} and of SFOLS are dual, they share the same optimal value.

\subsection{Optimal Unsupervised Skill Discovery}
\label{sec:skill-discovery}

Recently, \citet{Eysenbach+2022} studied the problem of unsupervised skill discovery~\cite{Gregor+2016,Eysenbach+2019,Hansen+2020,Liu&Abbeel2021} under a geometric perspective. They characterize policies by their discounted state occupancy measure, a $|\mathcal{S}|$-dimensional point lying on the probability simplex.
Formally, the discounted state occupancy measure of a policy $\pi$ is defined as:
\begin{equation}
    \rho^\pi(s) \equiv (1-\gamma) \sum_{t=0}^{\infty} \gamma^t P^\pi_t(s),
\end{equation}
where $P^\pi_t(s)$ is the probability that policy $\pi$ visits state $s$ at time $t$. 
Then, each policy can be represented as an $|\mathcal{S}|$-dimensional point $\vect{\rho}^\pi = [\rho^\pi(s_1) \ ... \ \rho^\pi(s_{|\mathcal{S}|})]^\top$. 
Since states are assumed to be discrete, any reward function, $r_{\vect{w}}$, can be represented as a vector $\vect{w} \in \mathbb{R}^{|\mathcal{S}|}$, where $\vect{w} = [r_{\vect{w}}(s_1) \ ... \ r_{\vect{w}}(s_{|\mathcal{S}|})]^\top$.
Based on this representation, the expected return of a policy, $v^\pi_{\vect{w}}$, can be expressed as the inner product between its state
occupancy measure and the reward vector. That is, $v^\pi_{\vect{w}} = \vect{\rho}^\pi \cdot \vect{w}$. \citet{Eysenbach+2022} then introduce two propositions, which we restate here:
\begin{proposition}
\cite{Eysenbach+2022}. For every state-dependent reward function, at least one policy that maximizes that reward function lies at a vertex of the discounted state occupancy polytope. That is, given any reward function, at least one of the corresponding optimal policies lies at a vertex of the polytope.
\end{proposition}
\begin{proposition}
\cite{Eysenbach+2022}. For every vertex $\vect{\rho}^\pi$ of the state occupancy polytope, there exists a reward function for which $\pi$ is optimal. That is, every vertex of the polytope is associated with the optimal policy of a reward function. \end{proposition}

Next, \cite{Eysenbach+2022} define the following open problem in the unsupervised skill discovery literature, which current state-of-the-art skill learning algorithms are unable to solve:
\begin{definition}
\label{def:vertex-discovery}
 \textit{(Vertex discovery problem \cite{Eysenbach+2022}) Given a controlled Markov process (i.e., an MDP without a reward function), find the smallest set of policies such that every vertex of the discounted state occupancy polytope contains at least one policy.} 
\end{definition}

% BRUNO2

\textbf{The importance of this open problem is the following. If solved, it would allow one to identify the smallest set of policies that solve \textit{all} possible tasks that may be defined over a controlled MDP. \footnote{The need for identifying the smallest set of policies is due to the fact that many tasks may share the same optimal policy.} Intuitively, then, identifying the (finite) set policies at the vertices of the state occupancy polytope allows one to solve \textit{all} possible tasks of interest.
}

Our key insight is that this problem is equivalent to the problem of finding a set of policies whose state occupancy measures form a CCS:
\begin{equation}
    \ccs = \{\vect{\rho}^\pi \ | \ \exists \vect{w}\ \text{s.t.}\ \forall \vect{\rho}^{\pi'}, \vect{\rho}^\pi \cdot \vect{w} \geq \vect{\rho}^{\pi'} \cdot \vect{w} \}.
\end{equation}

In the discrete state case, if the reward features $\vect{\phi}_t$ are a one-hot representation of the agent's state in $\mathbb{R}^{|\mathcal{S}|}$, then SFs correspond to the \textit{successor representation} (SR) \cite{Dayan1993}. Notably, the SR $\vect{\psi}^\pi \in \mathbb{R}^{|\mathcal{S}|}$ is also the  discounted state occupancy measure associated with policy $\pi$:\footnote{We omit the normalizing $(1-\gamma)$ term for clarity.} 
\begin{align}
    \vect{\psi}^\pi &\equiv \Ex_\pi \left[ \sum_{t=0}^{\infty} \gamma^t \vect{\phi}_t \right] \\
    &=  \sum_{t=0}^{\infty} \gamma^t \Ex_\pi \left[\vect{\phi}_t \right] \\
    &= \sum_{t=0}^{\infty} \gamma^t [P^\pi_t(s_1) \ ... \ P^\pi_t(s_{|\mathcal{S}|})]^\top \\
    &= [\rho^\pi(s_1) \ ... \ \rho^\pi(s_{|\mathcal{S}|})]^\top \\
    &\equiv \vect{\rho}^\pi .
\end{align}
Hence, SFOLS can be employed to solve the open problem stated by \citet{Eysenbach+2022} 
(Definition~\ref{def:vertex-discovery}).

Notice that although SFOLS solves this open problem, it may be challenging---in practice---to use it to identify the optimal set of skills. Recall that in the general case, the number of corner weights grows exponentially with $d$ (the number of objectives) \cite{Roijers2016}. This is generally feasible since the number of objectives is typically significantly smaller than the number of states. If using SFOLS to tackle the skill-discovery problem, however, $d$ is equal to the number of states. An interesting future direction is to exploit properties particular to the setting proposed by \cite{Eysenbach+2022} in order to reduce the complexity of SFOLS and employ it to identify optimal sets of skills.

\section{Experiments Details}
\label{ap:experiment}

The code containing the algorithms and training scripts necessary to reproduce the results are available at \url{https://github.com/LucasAlegre/sfols}.

\subsection{Environments}

All environments used in the experiments can be found in the MO-Gym library \cite{Alegre2022}.

\paragraph{Deep Sea Treasure.}
The Deep Sea Treasure is a classic MORL environment \cite{Vamplew+2011,vanMoffaert&Nowe2014,Abels+2019,Yang+2019}.
The agent's state at a given time step $t$ is its coordinates in the grid, $S_t = [x,y]$.
The action space consists of four directions the agent can move to, $\mathcal{A} = \{ \text{up},\text{down},\text{left},\text{right}\}$.
The first component of the feature/reward vector $\vect{\phi}(s,a,s') \in \mathbb{R}^2$ is the treasure value\footnote{We adopted the treasures values as defined in \citet{Yang+2019}.} (or zero if the agent is in a blank cell), as shown in the left of Figure~\ref{fig:dst}, and the second component is a time penalty of $-1$ in all states. 
The cells with treasures are also terminal states.
We considered a discount factor of $\gamma = 0.99$ in this domain.
There are ten different optimal values in this domain's CCS, each corresponding to a policy that reaches one of the ten treasures in the map.

\paragraph{Four Room.}
The Four Room domain \cite{Barreto+2017,Gimelfarb+2021} is defined by a grid of dimensions $13\times13$ containing four rooms separated by walls.
Each time step $t$, the agent occupies a cell and can move to one of the four directions $\mathcal{A} = \{ \text{up},\text{down},\text{left},\text{right}\}$. If the destination cell is a wall, then the agent remains in its current cell.
The grid contains 3 different types of objects the agent can pick up, as shown in the left of Figure~\ref{fig:fourroom}.
There are 4 instances of each type of object in the grid.
The state space consists of the concatenation of the agent's current x-y coordinates and a set of binary variables indicating whether or not each object has already been picked up: $\mathcal{S} = \{0,..,12\}^2 \times \{0,1\}^{12}$.
The features $\vect{\phi}(s,a,s') \in \{0,1\}^3$ are one-hot encoded vectors indicating the type of object present in the current cell.
If there are no objects in the cell, the features are zeroed. 
A special case is the goal cell at the upper-right of the map, which has all features activated and terminates the episode.
We considered a discount factor of $\gamma = 0.95$ in this domain.

\paragraph{Reacher.} 
We adapt the Reacher environment from PyBullet \cite{Ellenberger2018}, as done in \citet{Barreto+2017,Gimelfarb+2021,Nemecek&Parr2021}.\footnote{The code for this environment and Figure~\ref{fig:reacher} (left) were adapted from \citet{Gimelfarb+2021}.}
The agent's state space $\mathcal{S} \subset \mathbb{R}^4$ consists of the angles and angular velocities of the robotic arm's two joints.
The agent's initial state is the one shown in the left of Figure~\ref{fig:reacher}.
The action space, originally continuous, is discretized using 3 values per dimension corresponding to maximum positive (+1), negative (-1), and zero torque for each actuator. This results in a total of 9 possible actions: $\mathcal{A} = \{-1,0,+1\}^2$.
Each feature $\vect{\phi}(s,a,s') \in \mathbb{R}^4$ is computed as $\phi_i(s,a,s') = 1 - 4 \Delta(\text{target}_i), i=1...4$, where $\Delta(\text{target}_i)$ is the Euclidean distance of the tip of the robotic arm to the $i$-th target position.
We considered a discount factor of $\gamma = 0.9$ in this domain.

\subsection{Computing Corner Weights}
\label{sec:corner-weights}

Computing the corner weights (line 11 of Algorithm~\ref{alg:ols}) is an important step of SFOLS.
In Algorithm~\ref{algo:corner-weights} we present the algorithm used to compute the corner weights efficiently in each iteration. For more details on each of these steps, see Chapter~3 of \citet{Roijers2016}.

\begin{algorithm}[ht]
\label{algo:corner-weights}
\caption{Corner Weights \cite{Roijers2016}}
\begin{algorithmic}[1]
   \STATE {\bfseries Input:} New SF vector $\vect{\psi}^\pi$, current weight vector $\vect{w}$, current SF set $\Psi$.

  \STATE Let $\mathcal{W}_{del}$ be the set of obsolete weights removed from $Q$ in line 10 of Algorithm~\ref{alg:ols}
  
  \STATE Add $\vect{w}$ to $\mathcal{W}_{del}$
  
    \STATE $\mathcal{V}_{rel} \leftarrow \{\vect{\psi}^\pi | \vect{\psi}^\pi \in \argmax_{\vect{\psi}^\pi \in \Psi} \vect{\psi}^\pi \cdot \vect{w}' \text{ for at least one } \vect{w}' \in \mathcal{W}_{del}\}$

    \STATE $\mathcal{B}_{rel} \leftarrow$ the set of boundaries of the weight simplex $\mathcal{W}$ involved in any $\vect{w}' \in \mathcal{W}_{del}$ 
    
    \STATE $\mathcal{W}_{c} \leftarrow \{\}$
    
    \FOR{each subset $\mathcal{X}$ of $d-1$ elements from $\mathcal{V}_{rel} \cup \mathcal{B}_{rel}$}
        \STATE $\vect{w}_c \leftarrow$ the weight in $\mathcal{W}$ where $\vect{\psi}^\pi$ intersects with the vectors/boundaries in $\mathcal{X}$
        \STATE Add $\vect{w}_c$ to $\mathcal{W}_c$
    \ENDFOR
    
    \STATE {\bfseries return} $\mathcal{W}_c$
\end{algorithmic}
\end{algorithm}

\subsection{Additional Results}
\label{sec:dst-iterations}

In Figure~\ref{fig:dst-iterations}, we show the SF set, $\Psi$, and the GPI-expanded SF set , $\Psi^{\text{GPI}}$, computed at each iteration of SFOLS in the DST environment.
In the upper-left panel, we also show the variation of the hypervolume metric \cite{Hayes+2022} at each iteration.
Given a partial CCS, $\Psi$, and a reference point $\vect{\psi}_{\text{ref}}$, the hypervolume is defined as:
\begin{equation}
  \text{hypervolume}(\Psi, \vect{\psi}_{\text{ref}}) = \bigcup_{\vect{\psi}^\pi \in \Psi} \text{volume}(\vect{\psi}_{\text{ref}}, \vect{\psi}^\pi),  
\end{equation}
where $\text{volume}(\vect{\psi}_{\text{ref}}, \vect{\psi}^\pi)$ is the volume of the hypercube spanned by the reference vector, $\vect{\psi}_{\text{ref}}$, and the vector $\vect{\psi}^\pi$.
Notice that the hypervolume does not necessarily correlate with the mean value achieved over the space of reward weights $\mathcal{W}$. For instance, a set of diverse sub-optimal policies with low mean value can still have a high value of the hypervolume metric.

Notice that, from iteration $3$ onward, SFOLS+GPI reaches optimal performance for every weight vector, that is, $\Psi^{\text{GPI}} = \ccs$.
Meanwhile, the SF set only recovers the complete CCS at iteration 13.

\begin{figure*}[!h]
    \vskip 0.1in
	\begin{center}
		\begin{tabular}{cccc}
			\includegraphics[width=0.23\linewidth]{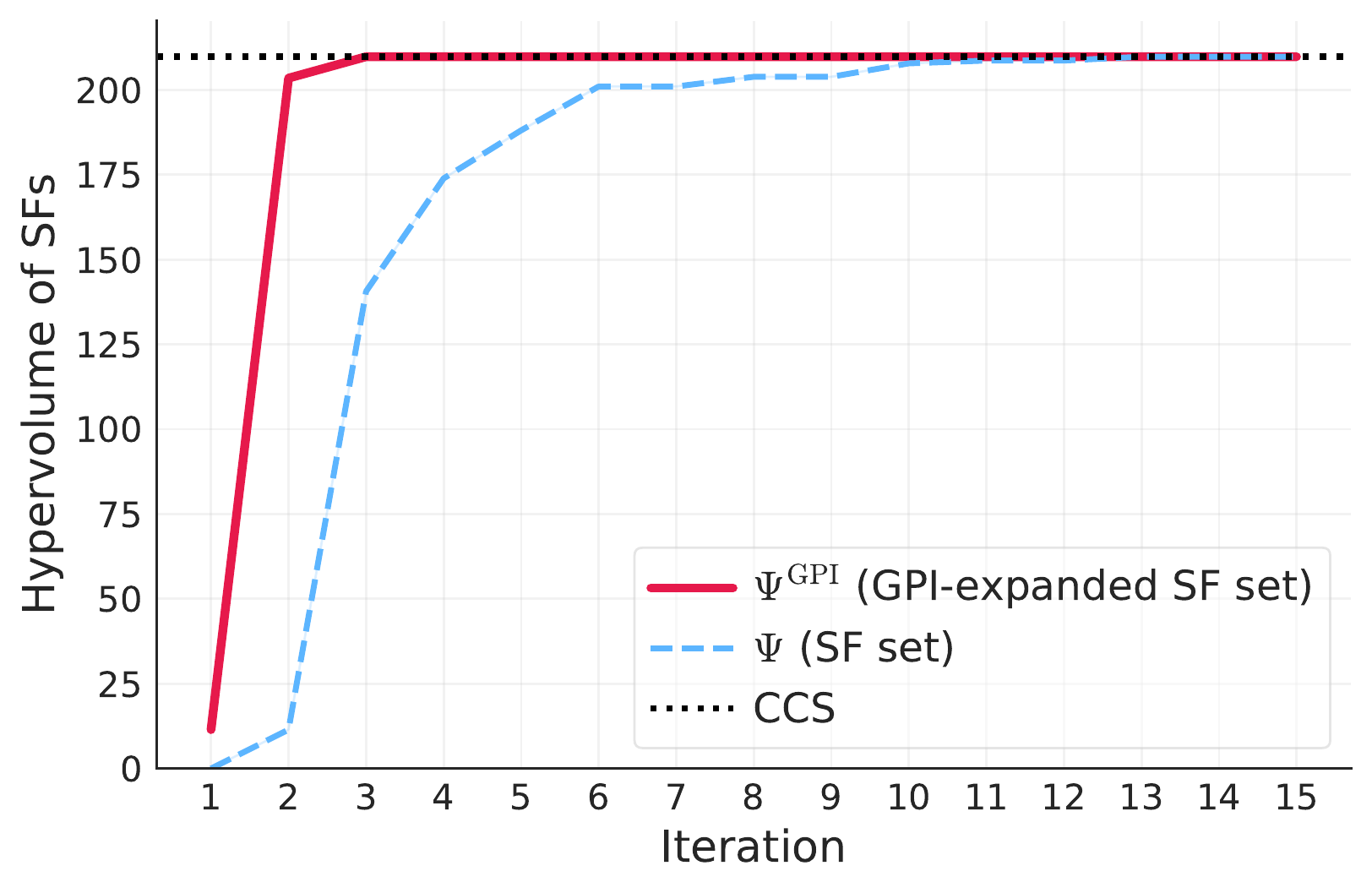} &
			\includegraphics[width=0.23\linewidth]{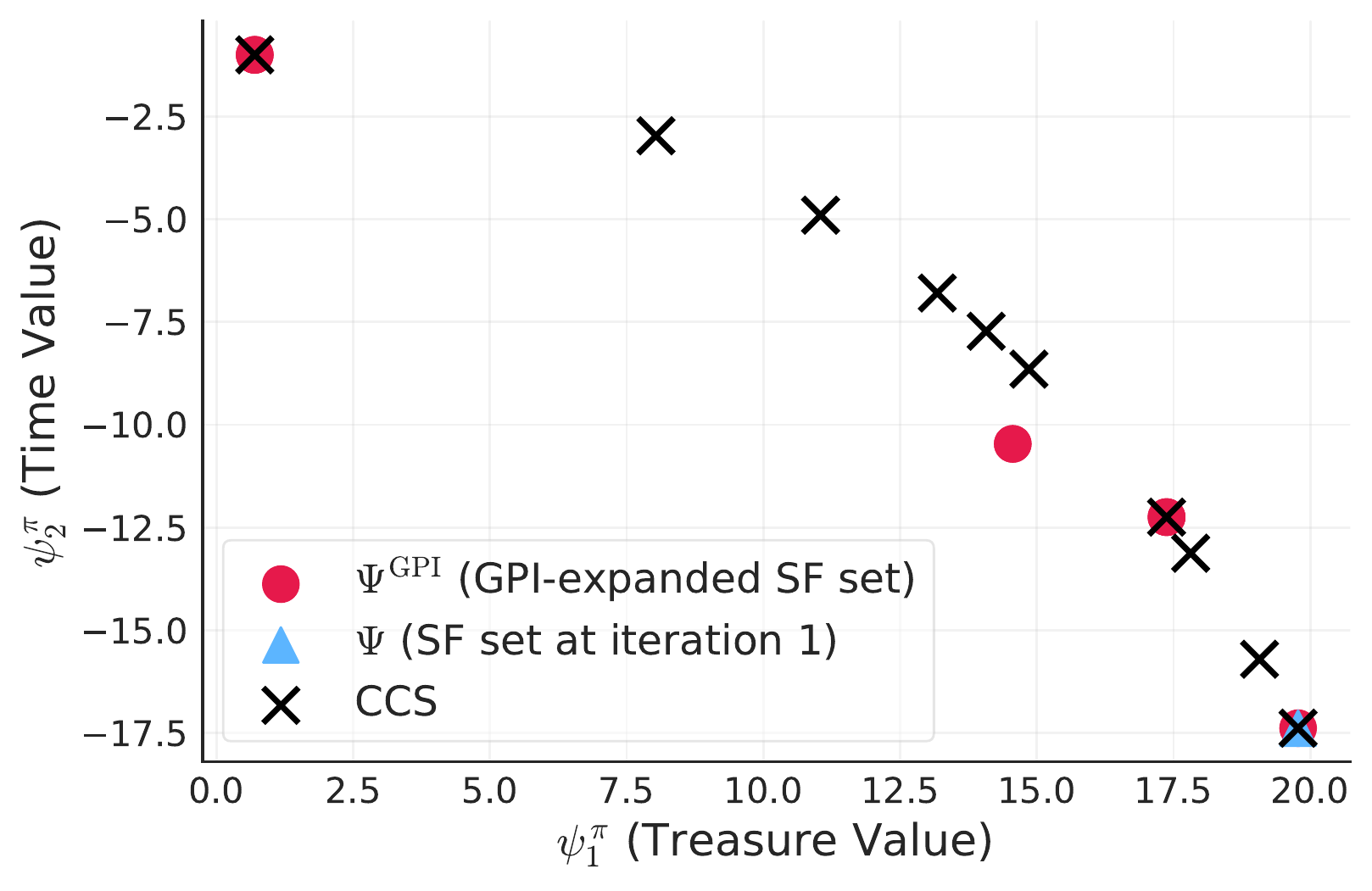} & \includegraphics[width=0.23\linewidth]{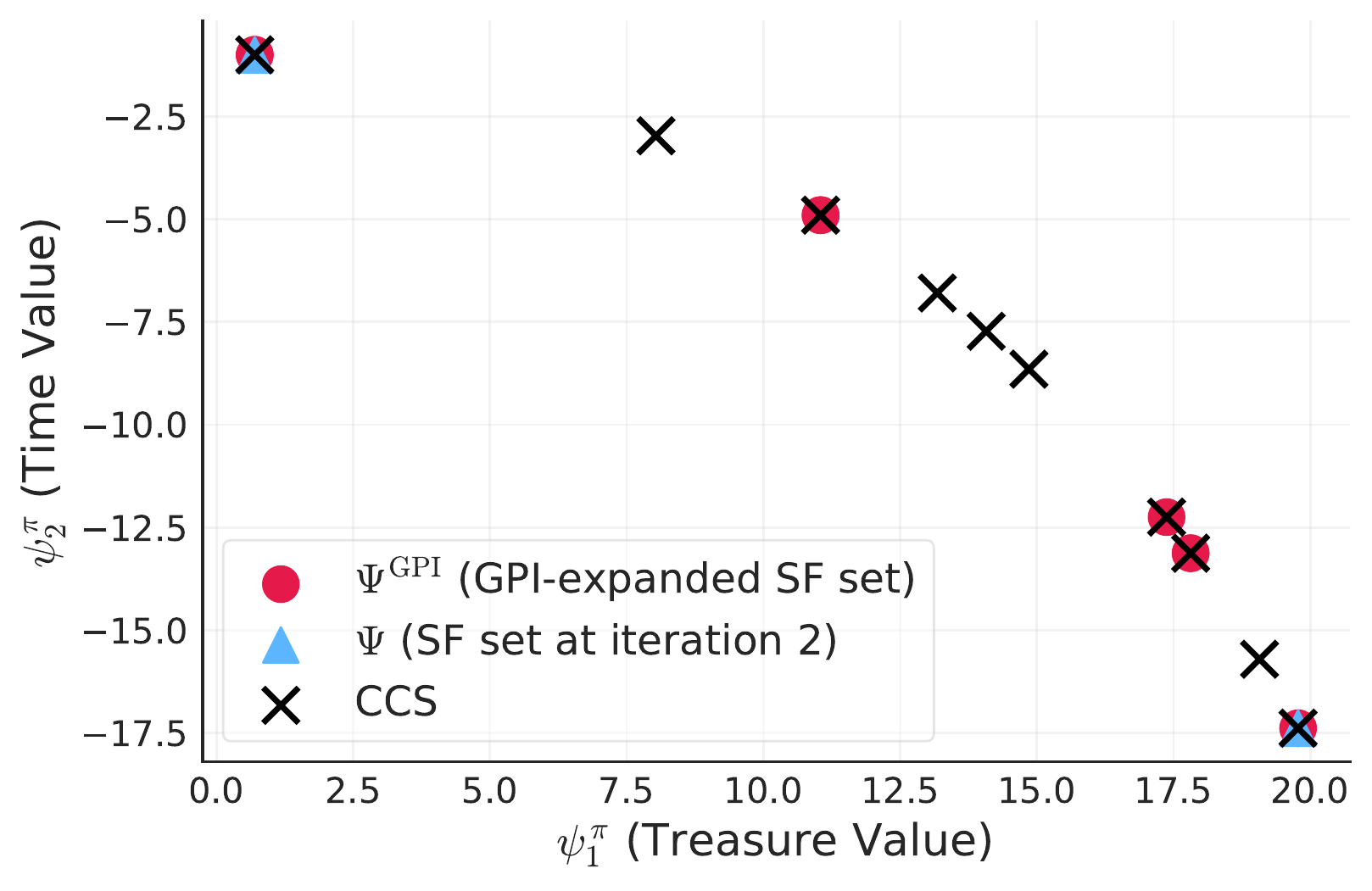} &
			\includegraphics[width=0.23\linewidth]{figs/ccs_dst3.pdf} \\
			\includegraphics[width=0.23\linewidth]{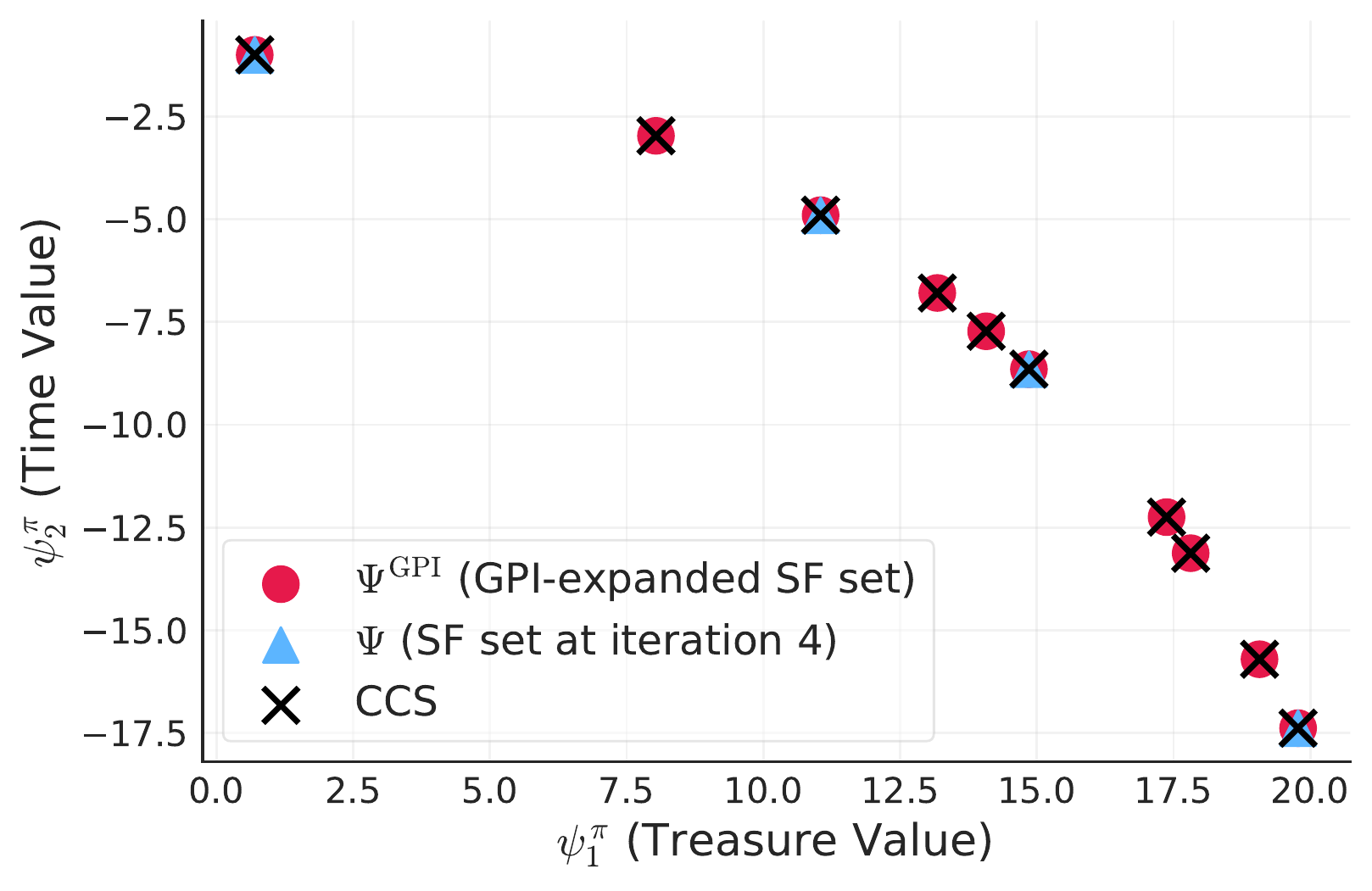} &
			\includegraphics[width=0.23\linewidth]{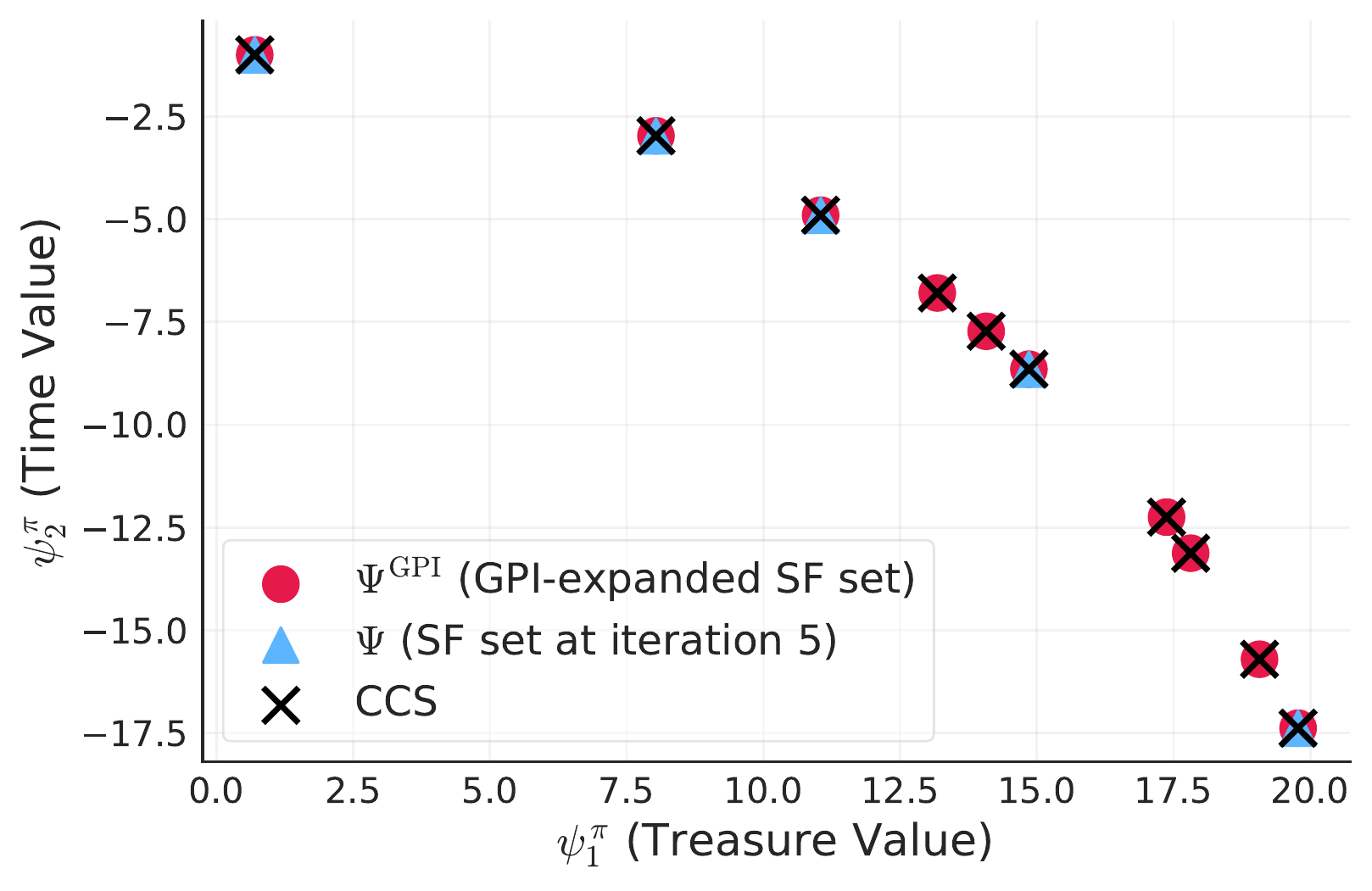} & \includegraphics[width=0.23\linewidth]{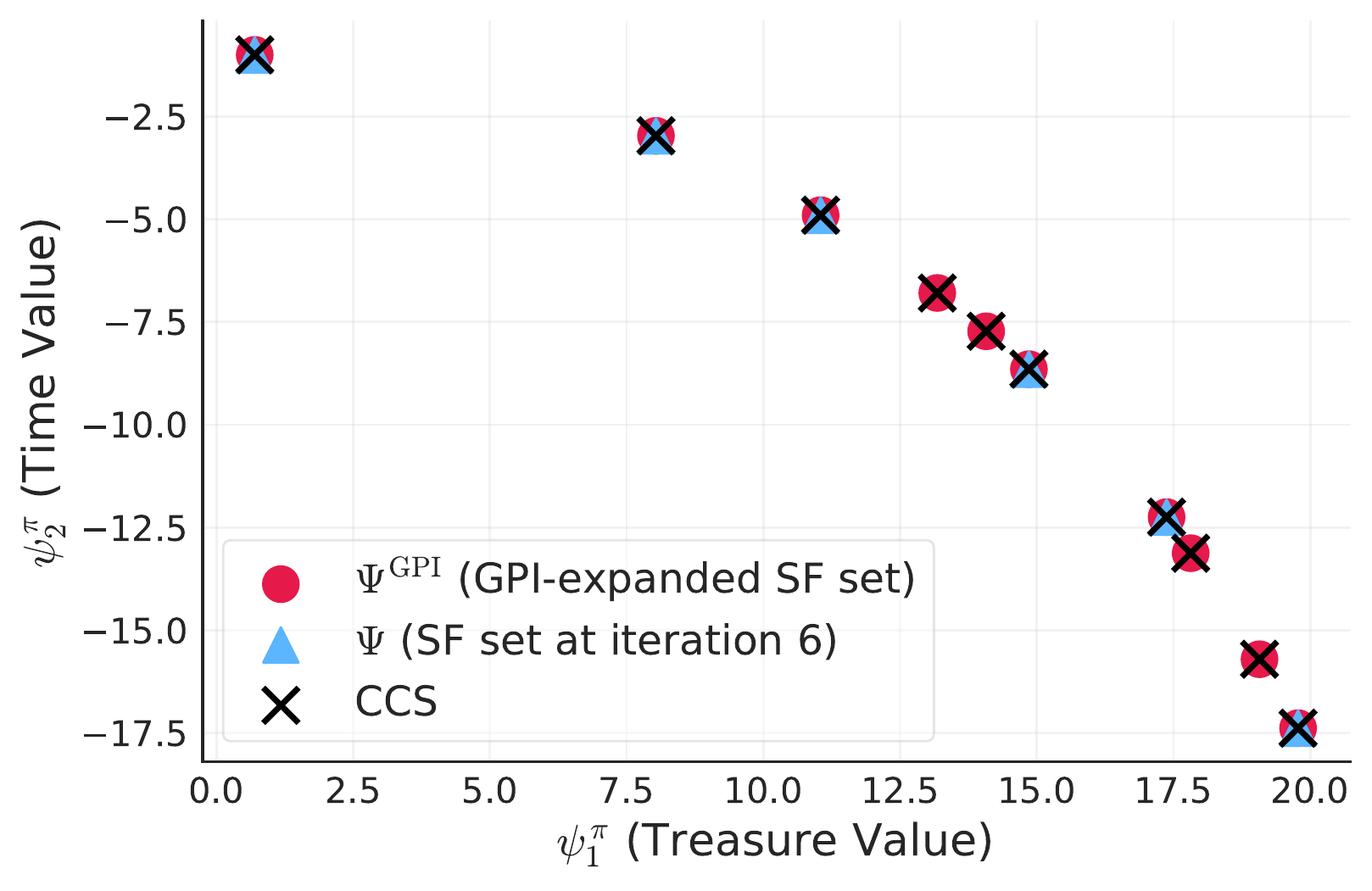} &
			\includegraphics[width=0.23\linewidth]{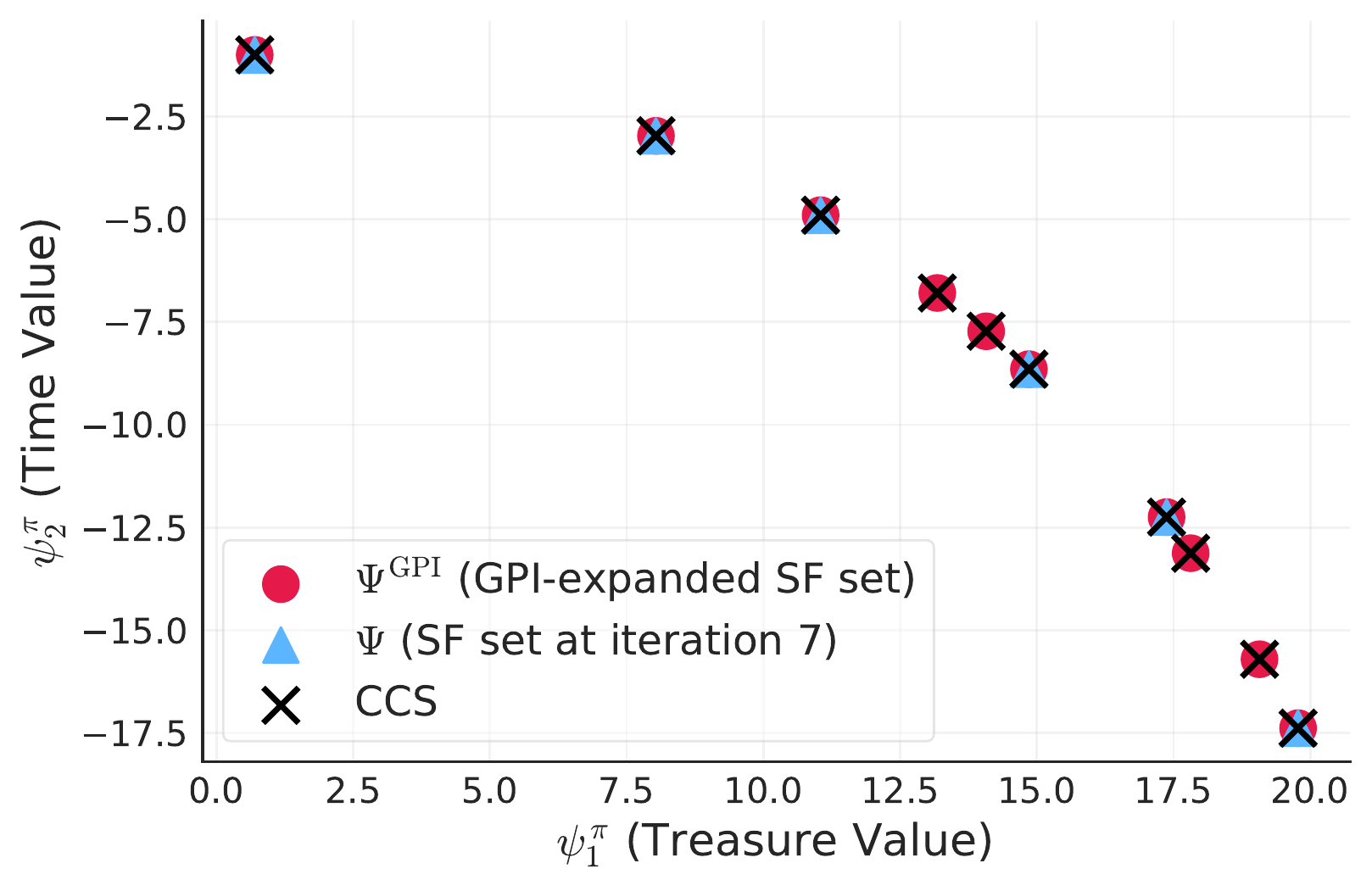} \\
			\includegraphics[width=0.23\linewidth]{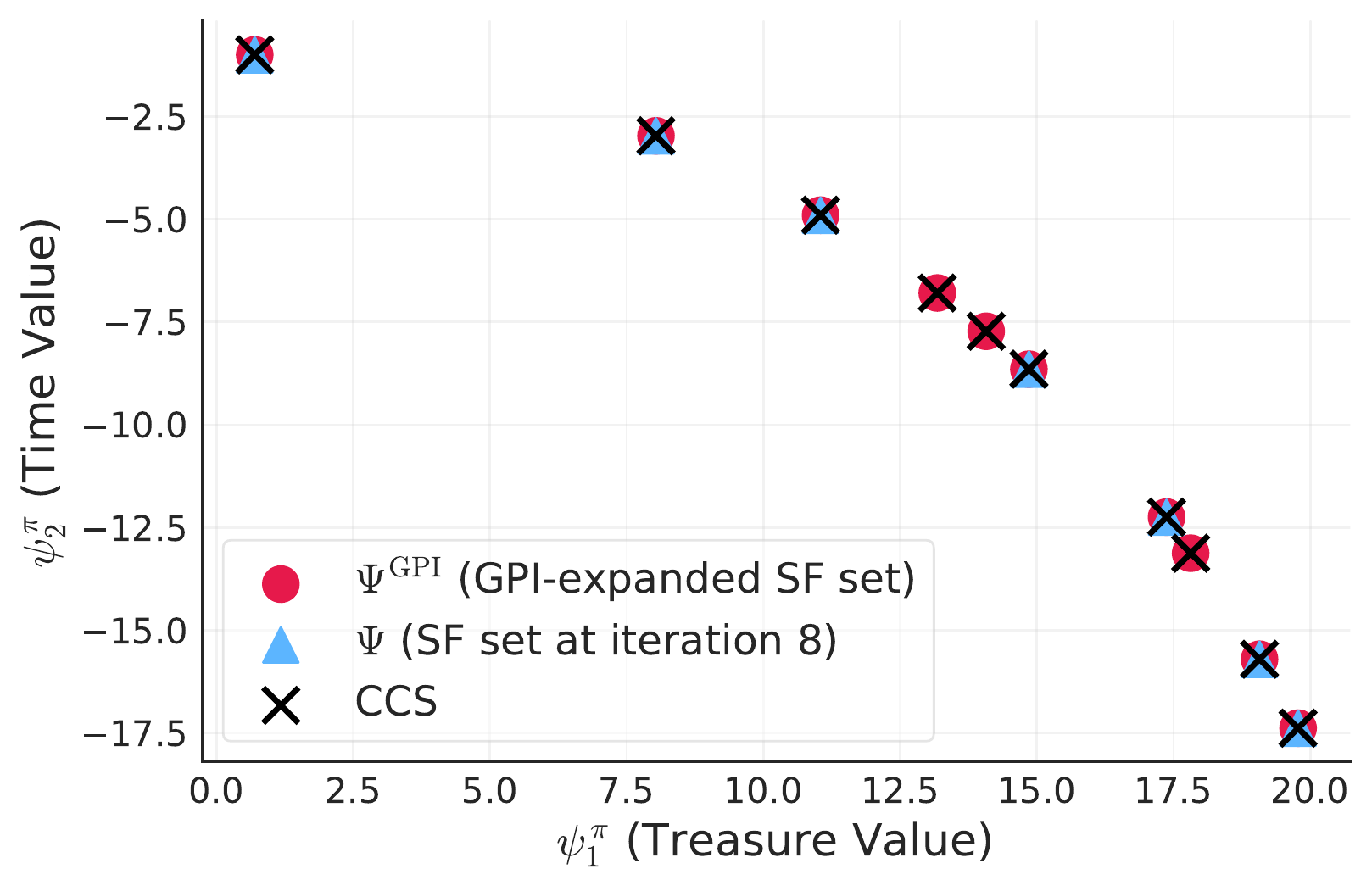} &
			\includegraphics[width=0.23\linewidth]{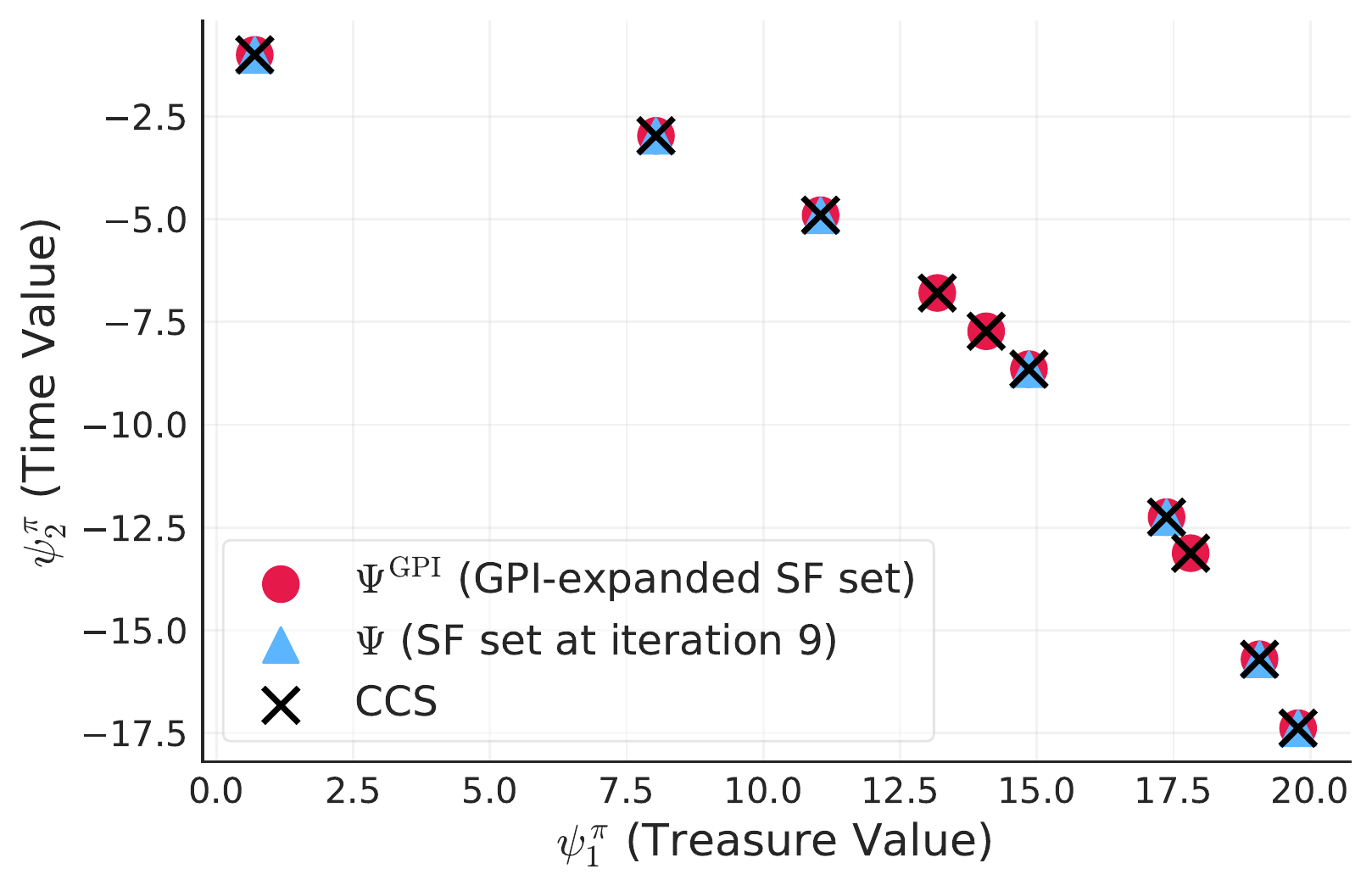} & \includegraphics[width=0.23\linewidth]{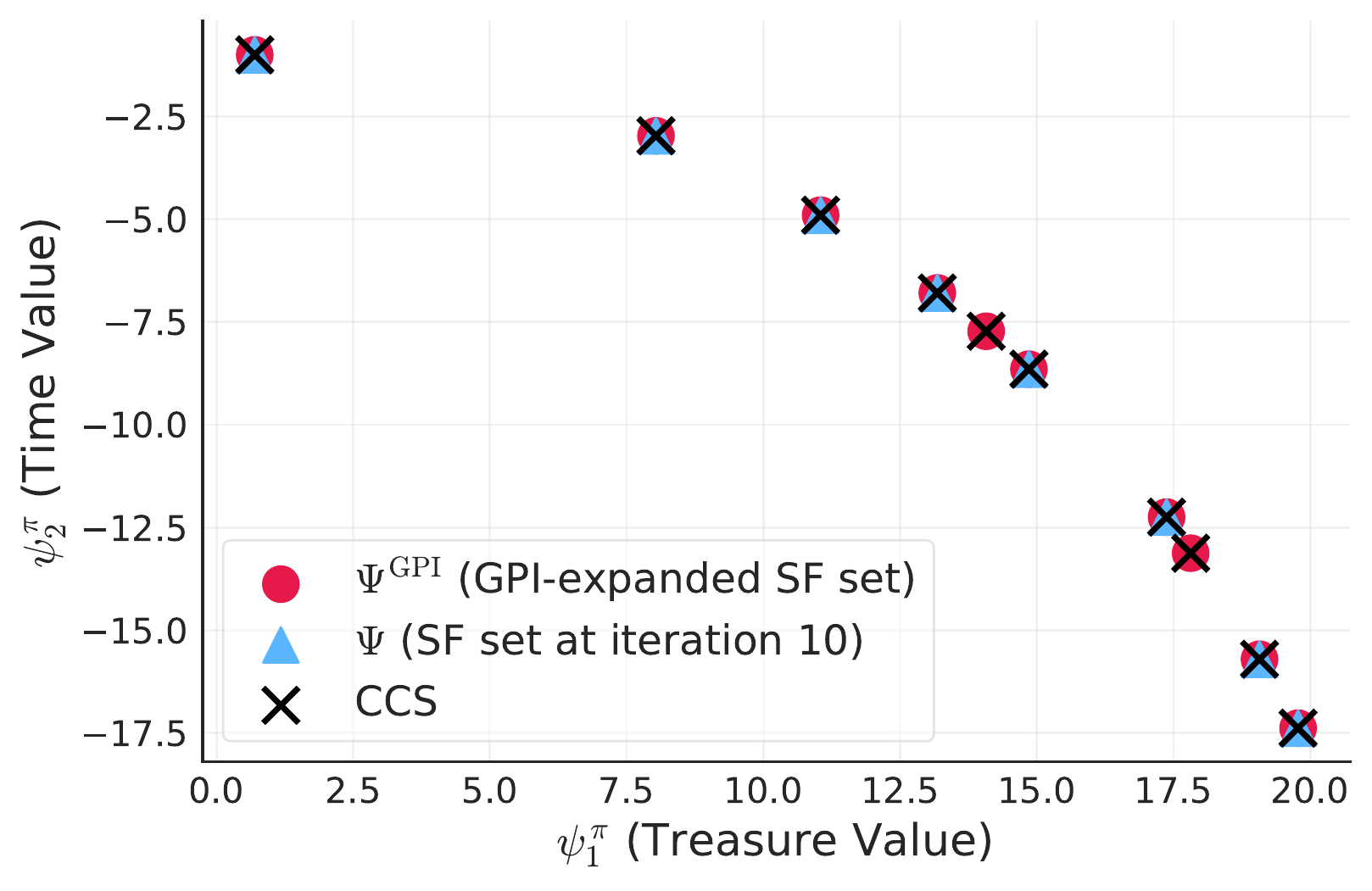} &
			\includegraphics[width=0.23\linewidth]{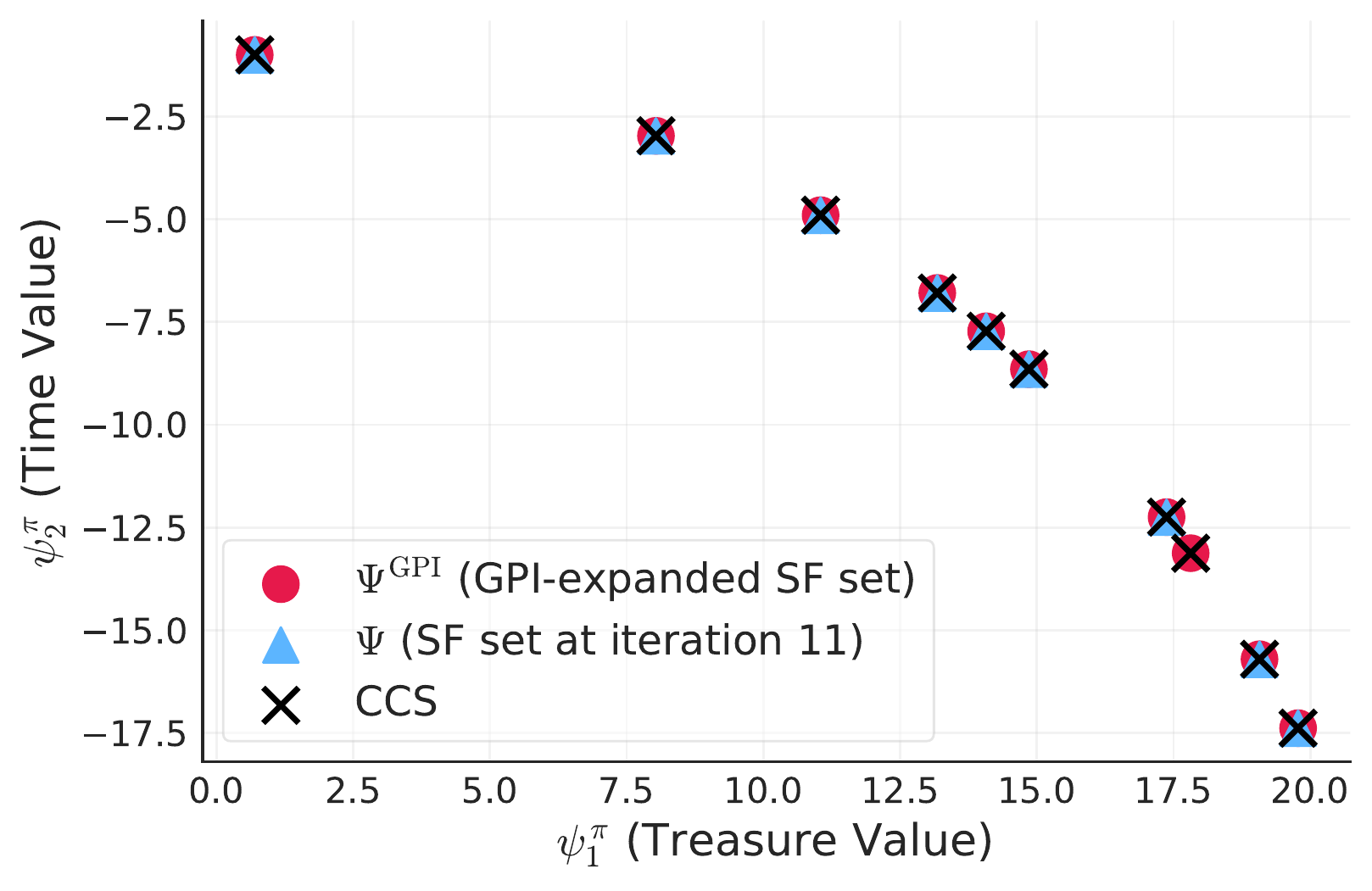} \\
			\includegraphics[width=0.23\linewidth]{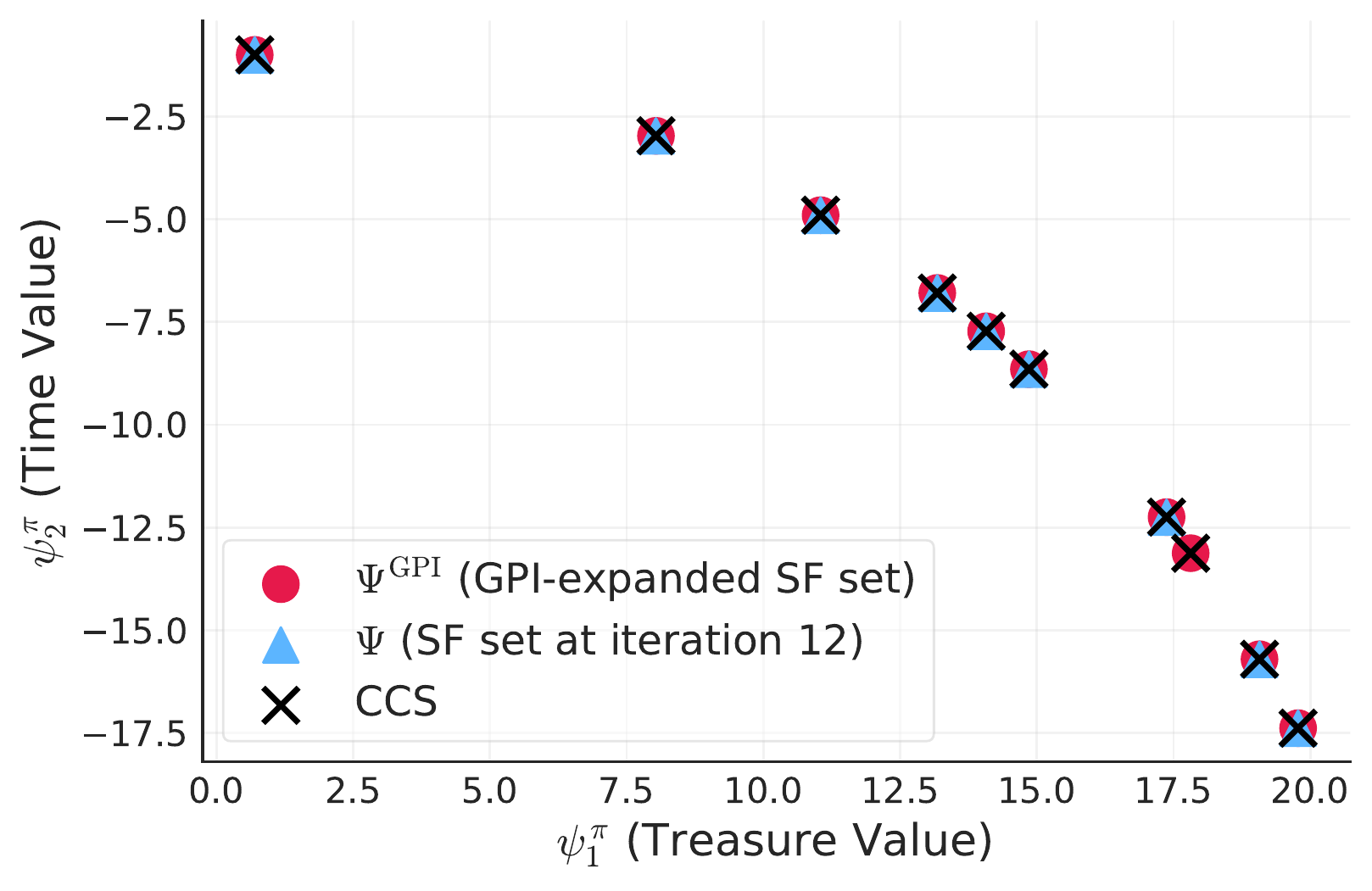} &
			\includegraphics[width=0.23\linewidth]{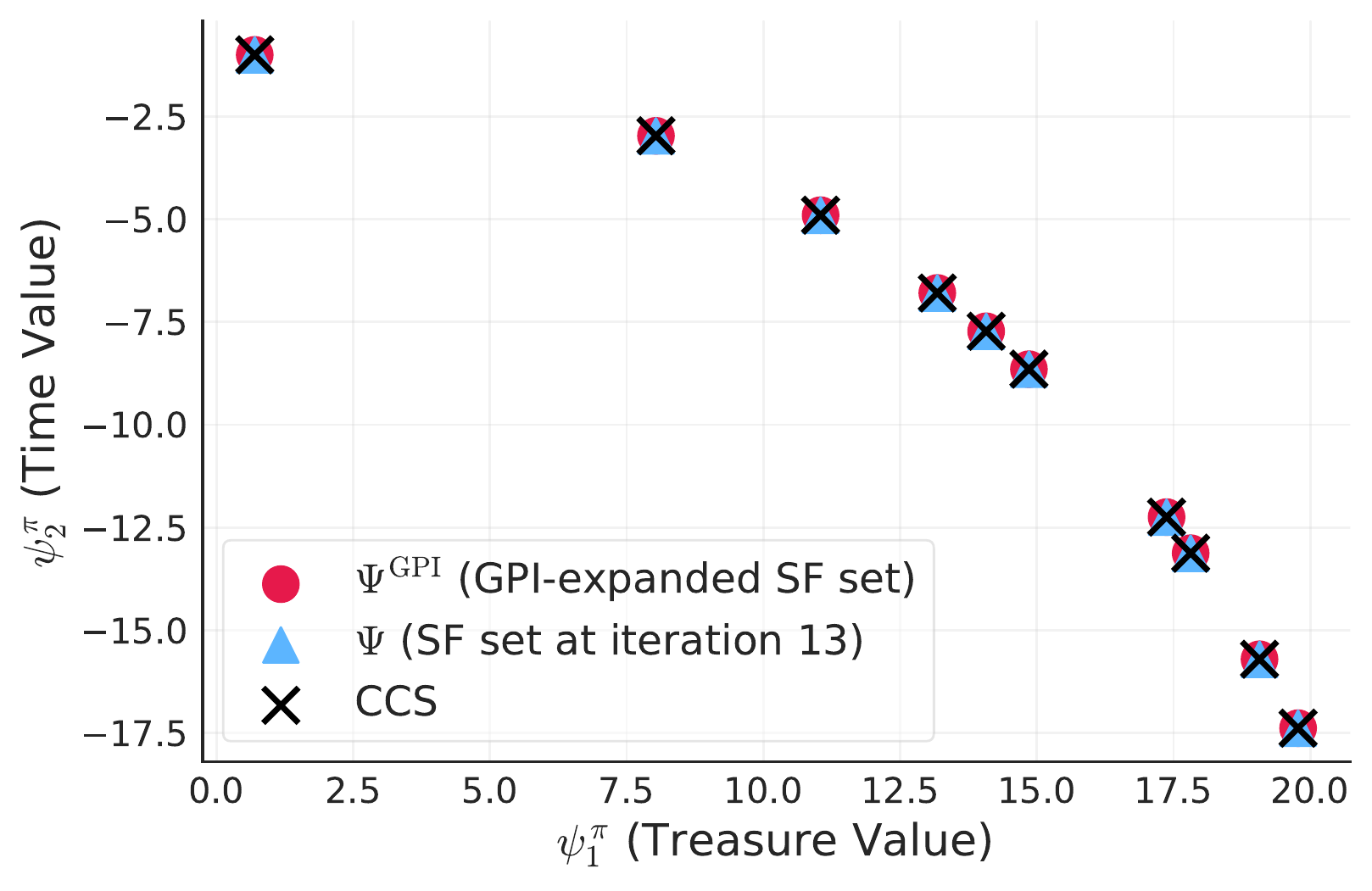} &
			\includegraphics[width=0.23\linewidth]{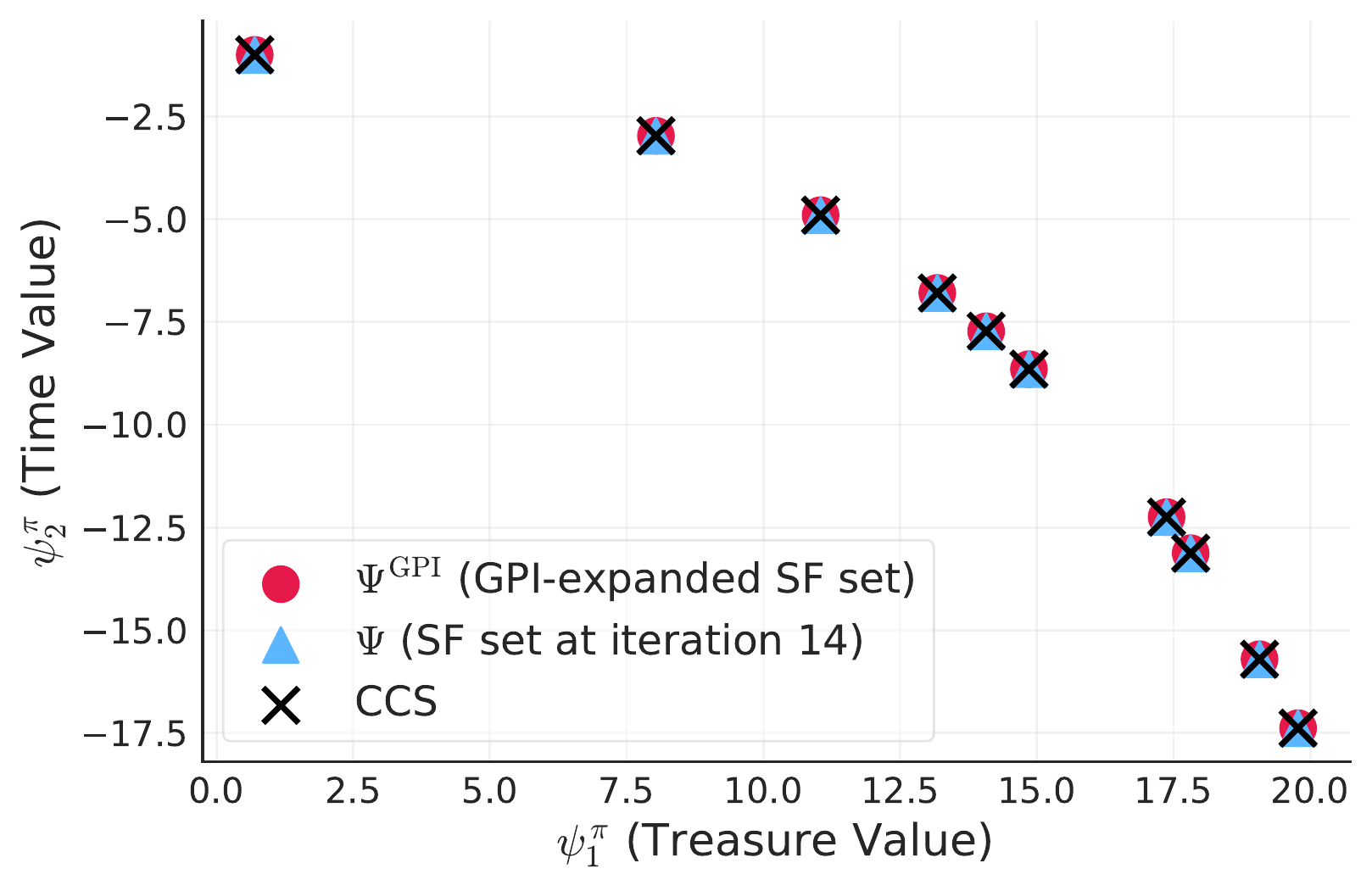} & \includegraphics[width=0.23\linewidth]{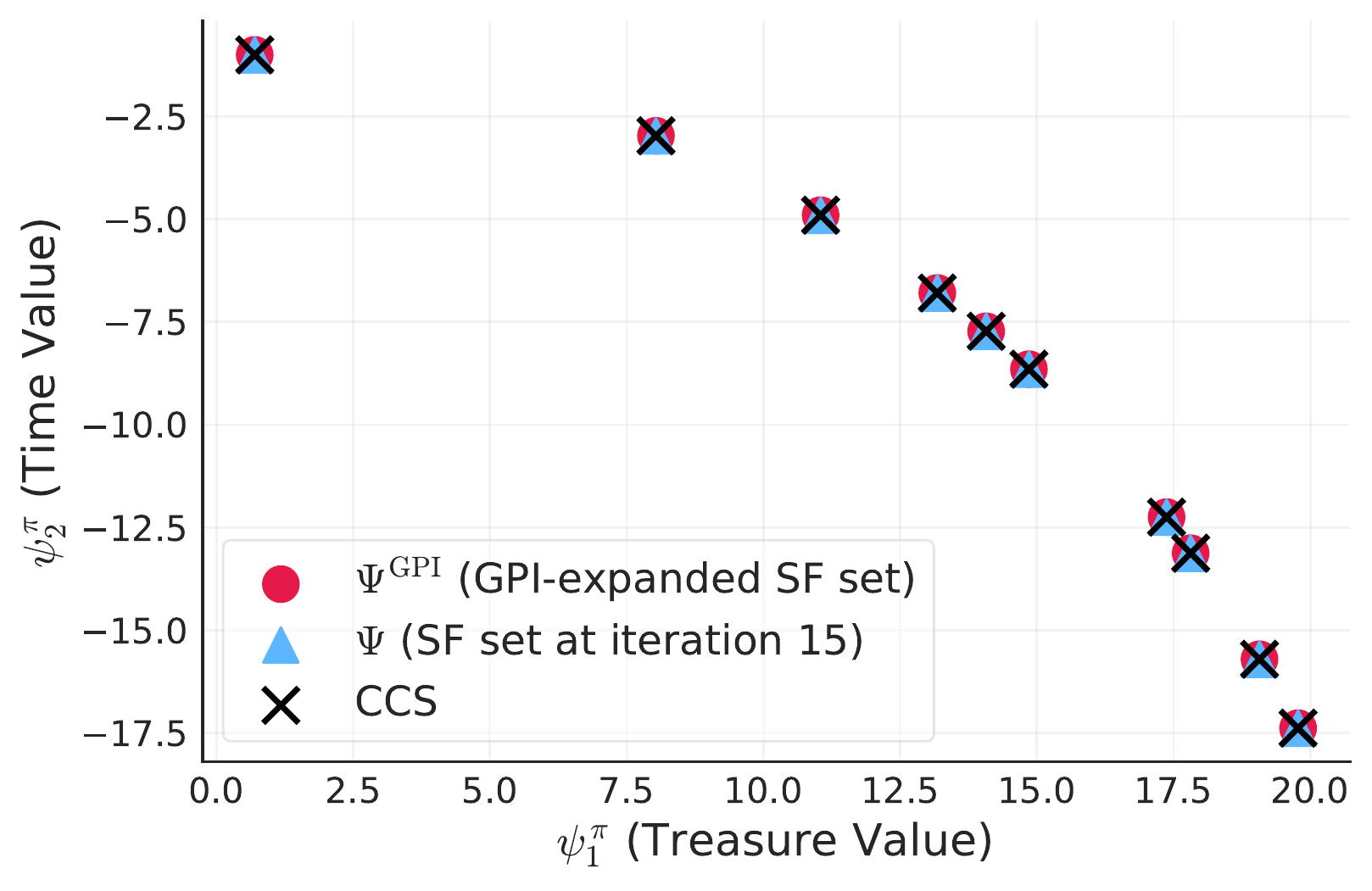}
		\end{tabular} 
	\end{center} 
	\caption{Hypervolume metric and the SF and GPI-expanded SF sets at each iteration of SFOLS in the DST domain. The reference point for the hypervolume calculation is $\vect{\psi}_{\text{ref}} = [0.0, -17.383]$.}
	\label{fig:dst-iterations}
	\vskip -0.1in
\end{figure*}

In order to provide a clearer visualization of the performance differences between each algorithm, in Figure~\ref{fig:dst-zoom} we show the curves, previously depicted in Figure~\ref{fig:dst}, separately for the SMP and GPI policies.
Similarly, in Figure~\ref{fig:fourroom-zoom} we separately zoom in on the curves previously depicted in Figure~\ref{fig:fourroom}.
Notice that SFOLS performance has significantly lower variance than the competing algorithms. This is because WCPI and the random baseline depend on random initializations of the reward vectors, while SFOLS does not.

\begin{figure*}[h!]
\vskip 0.1in
\begin{center}
\centerline{
\includegraphics[width=0.35\linewidth,align=c]{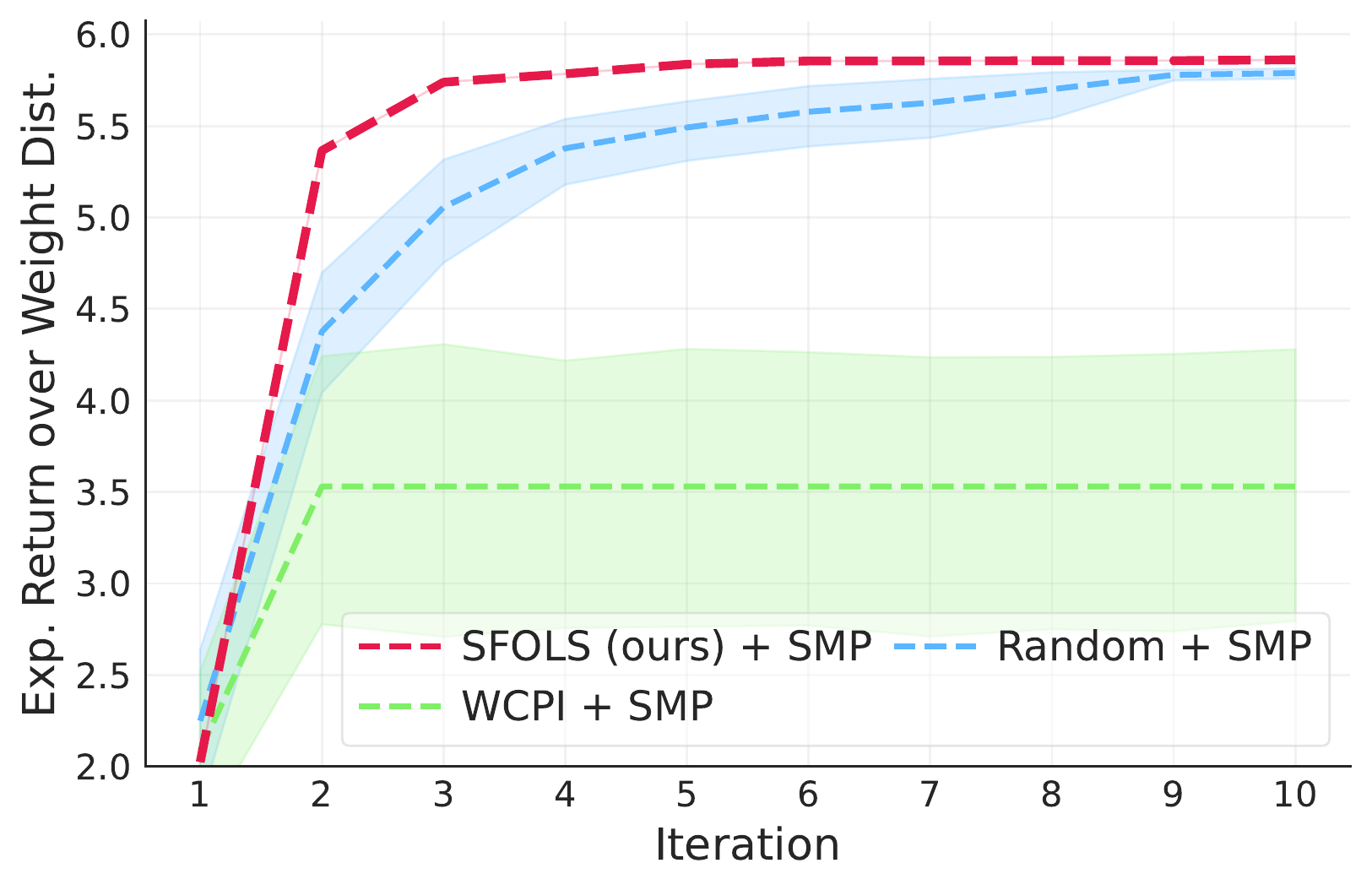}
\includegraphics[width=0.35\linewidth,align=c]{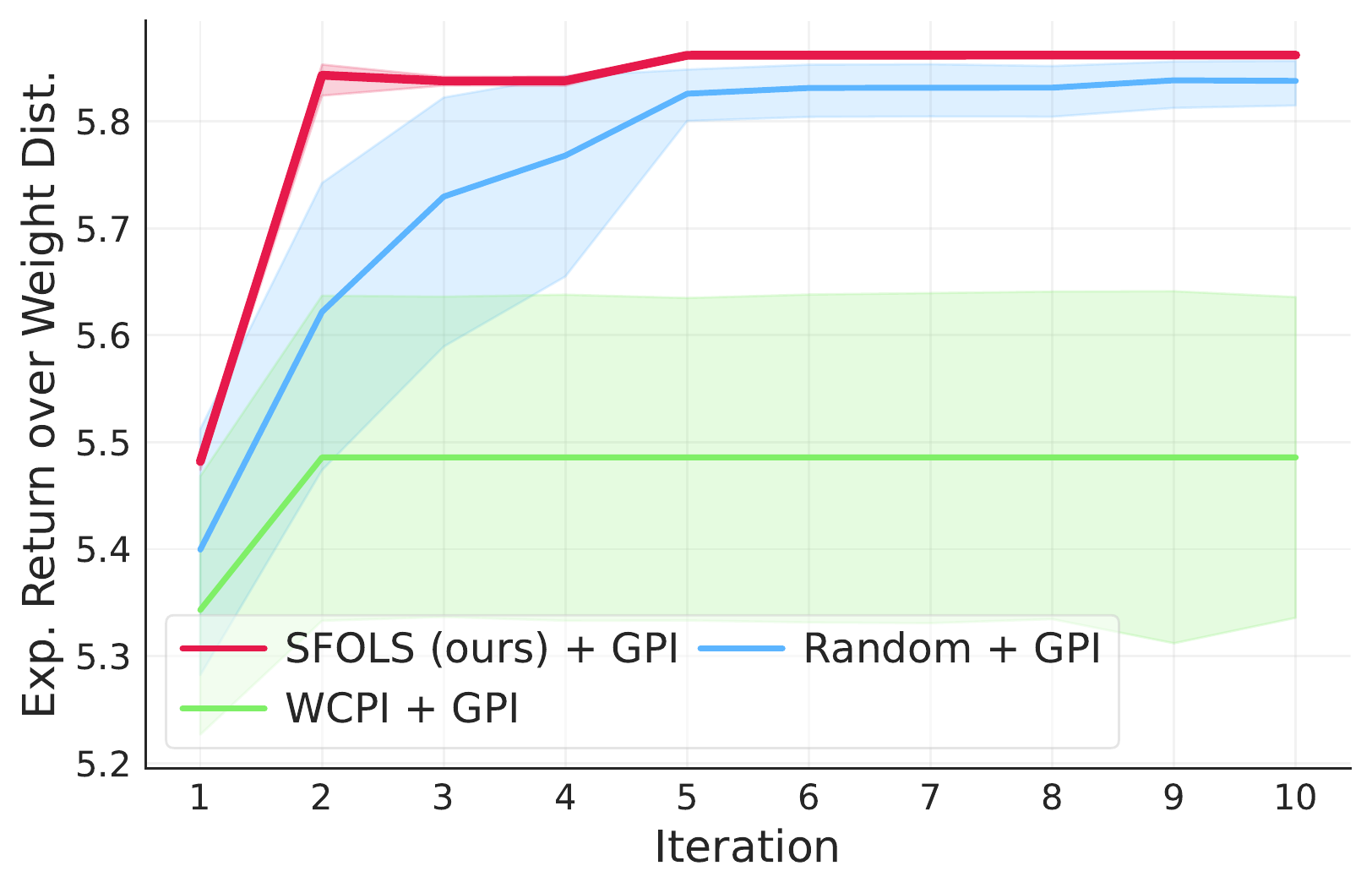}}
\caption{Expected return of each algorithm over the task distribution, $\mathcal{W}$, when evaluated using either SMP (\textbf{left}) or GPI (\textbf{right}) in the DST domain.}
\label{fig:dst-zoom}
\end{center}
\vskip -0.1in
\end{figure*}
\begin{figure*}[h!]
\vskip 0.0in
\begin{center}
\centerline{
\includegraphics[width=0.35\linewidth,align=c]{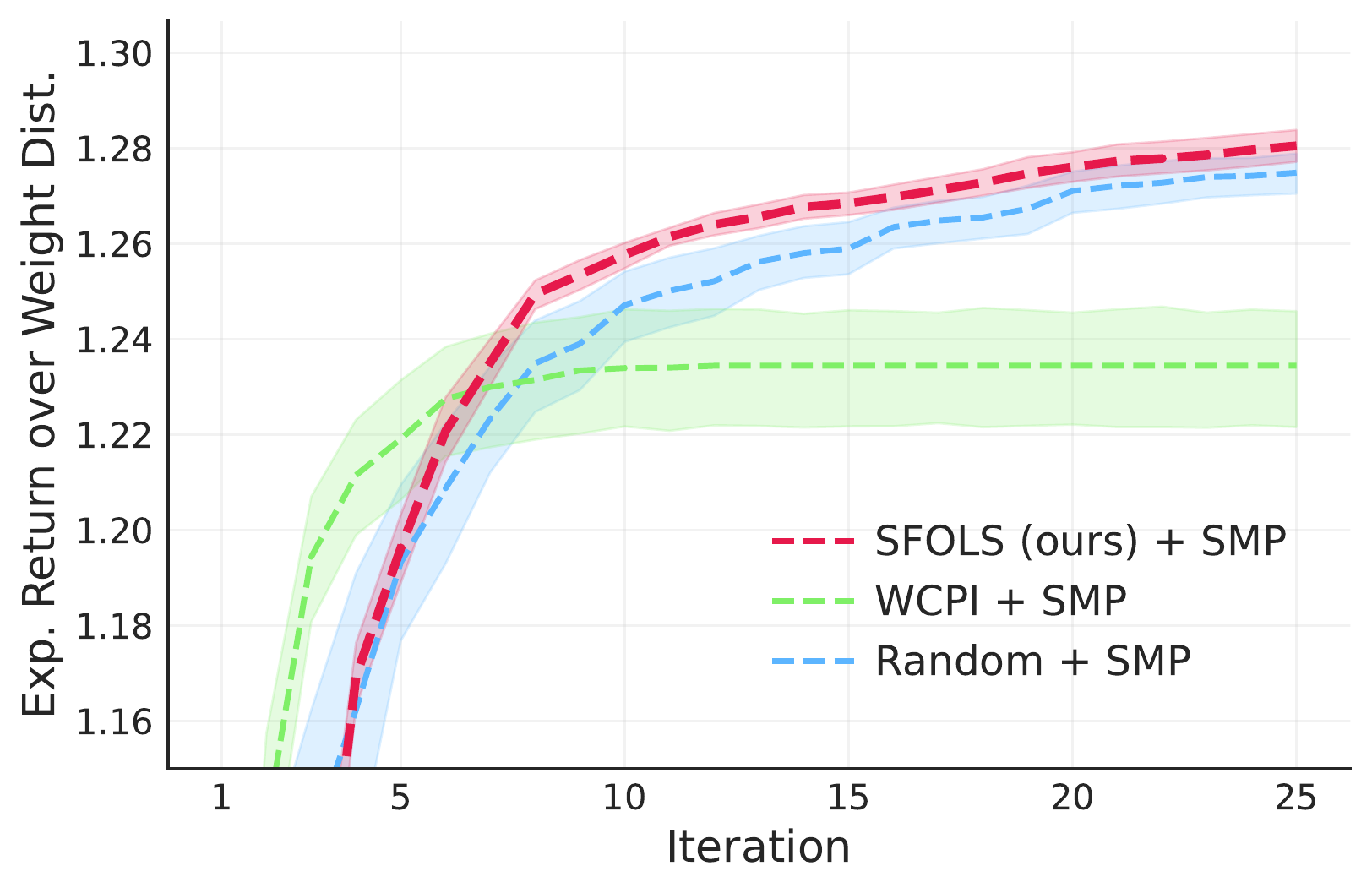}
\includegraphics[width=0.35\linewidth,align=c]{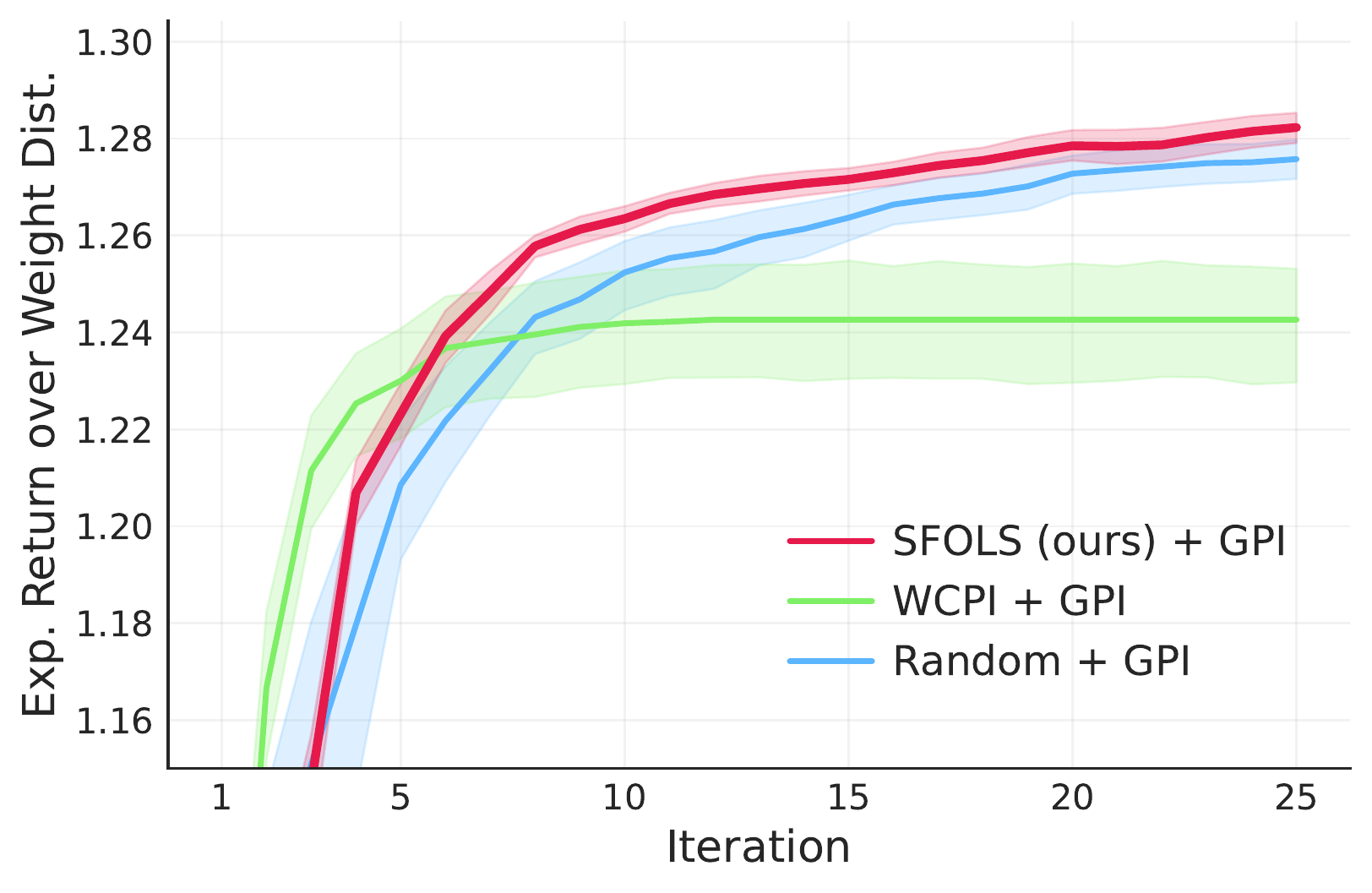}}
\caption{Expected return of each algorithm over the task distribution, $\mathcal{W}$, when evaluated using either SMP (\textbf{left}) or GPI (\textbf{right}) in the Four Room domain.}
\label{fig:fourroom-zoom}
\end{center}
\vskip -0.2in
\end{figure*}

\subsection{Learning SFs}

In Algorithm~\ref{alg:tabularsf} we introduce an algorithm to learn SFs using GPI based on tabular Q-learning.
We used a learning rate of $\alpha = 0.3$, and $\varepsilon$-greedy exploration with the value of $\varepsilon$ linearly decaying from $1$ to $0.05$.
At each iteration, the total number of time steps used to learn the SFs was set to $\text{num\_steps} = 10^5$ for the DST domain and $\text{num\_steps} = 10^6$ for the Four Room domain.

\begin{algorithm}[ht]
\caption{Learn New Tabular Successor Features with GPI}
\label{alg:tabularsf}
\begin{algorithmic}[1]
\STATE {\bfseries Input:}  SF set $\Psi = \{\vect{\psi}^{\pi_1},...,\vect{\psi}^{\pi_{n-1}}\}$, learning rate $\alpha$, exploration prob. $\varepsilon$, new preference weight $\vect{w}$, $\text{num\_steps}$

\STATE Initialize new SF: $\vect{\psi}^{\pi_n}(s,a) \leftarrow \vect{\psi}^{\pi_i}(s,a) \forall (s,a)$, with $i = \argmax_i \vect{\psi}^{\pi_i} \cdot \vect{w}$

\STATE $\text{new\_episode} \leftarrow \text{True}$

\FOR{$t = 0 ... \text{num\_steps}$}
\IF{$\text{new\_episode}$}
    \STATE $S_t \leftarrow \text{initial state sampled from } \mu$
    \STATE $\text{new\_episode} \leftarrow \text{False}$
\ENDIF

\IF{$\text{Bernoulli}(1-\varepsilon)$}
\STATE $A_t \leftarrow \argmax\nolimits_a \max\nolimits_{i} \vect{\psi}^{\pi_i}(S_t,a) \cdot \vect{w}$ \COMMENT{GPI}
\ELSE
    \STATE $A_t \leftarrow \text{random action sampled from Uniform}(\mathcal{A})$ \COMMENT{$\varepsilon$-greedy exploration}
\ENDIF

\STATE Execute $A_t$, observe $\vect{\phi}_t$ and $S_{t+1}$

\IF{$S_{t+1}$ is terminal}
\STATE $\text{new\_episode} \leftarrow \text{True}$
\STATE $\vect{\delta}_t = \vect{\phi}_t - \vect{\psi}^{\pi_n}(S_t,A_t)$
\ELSE
\STATE $a' \leftarrow \argmax\nolimits_{a'} \max\nolimits_{i} \vect{\psi}^{\pi_i}(S_{t+1},a') \cdot \vect{w}$
\STATE $\vect{\delta}_t = \vect{\phi}_t + \gamma \vect{\psi}^{\pi_n}(S_{t+1}, a')  - \vect{\psi}^{\pi_n}(S_t,A_t)$
\ENDIF

\STATE $\vect{\psi}^{\pi_n}(S_t,A_t) \leftarrow \vect{\psi}^{\pi_n}(S_t,A_t) + \alpha \vect{\delta}_t$ \COMMENT{Update SFs}

\ENDFOR
\STATE {\bfseries return} $\vect{\psi}^{\pi_n}$

\end{algorithmic}
\end{algorithm}

In the continuous state case, we learned SFs using neural networks as detailed in Algorithm~\ref{alg:sfdqn}.
We employed a training scheme similar to DQN \cite{Mnih+2015}.
Each SF $\vect{\psi}(s,a)$ is a multi-layer perceptron (MLP) neural network with two layers of 256 neurons and ReLU non-linear activations. Each neural network outputs a matrix $\mathbb{R}^{|\mathcal{A}| \times d}$, consisting of the values for each one of the $|\mathcal{A}|$ actions and $d$ features. 
We also adopted popular DQN extensions to speed up and stabilize learning, such as double Q-learning \cite{vanHasselt+2016} and prioritized experience replay \cite{Schaul+2016,Fujimoto+2020}.
We used Adam \cite{Kingma&Ba2015} with a learning rate of $0.001$ as the gradient-based optimizer, and a mini-batch of size $b=256$.
The value of $\varepsilon$ was fixed to $0.05$ for the $\varepsilon$-greedy exploration.
We trained each SF for $\text{num\_steps} = 2\cdot10^5$ time steps.

\begin{algorithm}[ht]
\caption{Learn New DQN-based Successor Features with GPI}
\label{alg:sfdqn}
\begin{algorithmic}[1]
\STATE {\bfseries Input:}  SF set $\Psi = \{\vect{\psi}^{\pi_1},...,\vect{\psi}^{\pi_{n-1}}\}$, replay buffer $\mathcal{D}$, exploration prob. $\varepsilon$, new preference weight $\vect{w}$, $\text{num\_steps}$

\STATE Initialize new SF $\vect{\psi}^{\pi_n}(s,a)$ as a neural network

\STATE $\text{new\_episode} \leftarrow \text{True}$

\FOR{$t = 0 ... \text{num\_steps}$}
\IF{$\text{new\_episode}$}
    \STATE $S_t \leftarrow \text{initial state sampled from } \mu$
    \STATE $\text{new\_episode} \leftarrow \text{False}$
\ENDIF

\IF{$\text{Bernoulli}(1-\varepsilon)$}
\STATE $A_t \leftarrow \argmax\nolimits_a \max\nolimits_{i} \vect{\psi}^{\pi_i}(S_t,a) \cdot \vect{w}$ \COMMENT{GPI}
\ELSE
    \STATE $A_t \leftarrow \text{random action sampled from Uniform}(\mathcal{A})$ \COMMENT{$\varepsilon$-greedy exploration}
\ENDIF

\STATE Execute $A_t$, observe $\vect{\phi}_t$ and $S_{t+1}$
\STATE Add $(S_t,A_t,\vect{\phi}_t,S_{t+1})$ to $\mathcal{D}$
\IF{$S_{t+1}$ is terminal}
\STATE $\text{new\_episode} \leftarrow \text{True}$
\ENDIF

\STATE Sample mini-batch $\{(s_i,a_i,\vect{\phi}_i,s'_i)\}_{i=1}^{b}$ from $\mathcal{D}$
\STATE $a'_i = \argmax_{a'} \max_i \vect{\psi}^{\pi_i}(s'_i,a')$
\STATE $\vect{y}_i = \vect{\phi}_i + \gamma  \vect{\psi}^{\pi_i}(s'_i,a'_i)$

\STATE Update SF by minimizing the loss $\mathcal{L}(\vect{\psi}^{\pi_n}) = \frac{1}{b}\sum_{i=1}^{b} \left[ (\vect{y}_i - \vect{\psi}^{\pi_n}(s_i,a_i))^2 \right]$

\ENDFOR
\STATE {\bfseries return} $\vect{\psi}^{\pi_n}$

\end{algorithmic}
\end{algorithm}

\subsection{Worst-Case Reward Policy Iteration \cite{Zahavy+2021}}

In Algorithm~\ref{algo:worst-case} we present the WCPI algorithm proposed by \citet{Zahavy+2021} to learn a diverse set of policies using SFs.
For details on how to compute the worst-case reward (line 5) by solving linear programs, see Lemma 4 of \citet{Zahavy+2021}.

\begin{algorithm}[ht]
\label{algo:worst-case}
\caption{SMP Worst Case Policy Iteration \cite{Zahavy+2021}}
\begin{algorithmic}[1]
   \STATE {\bfseries Initialize:} $\Pi \leftarrow \{\}$; $\Psi \leftarrow \{\}$; $\bar{\vect{w}} \leftarrow \text{random weight sampled from } \mathcal{W}$
   \STATE $\pi, \vect{\psi}^\pi \leftarrow \text{solution of the RL task } \bar{\vect{w}}$
   \STATE Add $\pi$ to $\Pi$ and $\vect{\psi}^\pi$ to $\Psi$
   \REPEAT
        \STATE $\bar{\vect{w}} \leftarrow \argmin_{\vect{w} \in \mathcal{W}} \max_{\pi \in \Pi} \vect{\psi}^\pi \cdot \vect{w}$
        \STATE $\pi, \vect{\psi}^\pi  \leftarrow \text{solution of the RL task } \bar{\vect{w}}$
        \STATE Add $\pi$ to $\Pi$ and $\vect{\psi}^\pi$ to $\Psi$
   \UNTIL{$v^{\text{SMP}}_{\bar{\vect{w}}}$ does not improve}
   
   \STATE{\bfseries return $\Pi, \Psi$}
\end{algorithmic}
\end{algorithm}

\subsection{Set of Independent Policies \cite{Alver&Precup2021}}

In Algorithm~\ref{algo:sip} we present the SIP algorithm proposed by \citet{Alver&Precup2021} to learn a set of independent policies using SFs.
Importantly, it requires that the features $\vect{\phi}(s,a,s') \in \mathbb{R}^d$ are \textit{independent} (see Definition 1 of \citet{Alver&Precup2021}) and that the MDP's transition function is deterministic.

\begin{algorithm}[ht]
\label{algo:sip}
\caption{Set of Independent Policies \cite{Alver&Precup2021}}
\begin{algorithmic}[1]
   \STATE {\bfseries Initialize:} $\Pi \leftarrow \{\}$; $\Psi \leftarrow \{\}$; $i \leftarrow 1$
   \WHILE{$i \leq d$}
        \STATE $\vect{w}_i \leftarrow$ weight with a positive value in the $i$-th component and negative values elsewhere
        \STATE $\pi, \vect{\psi}^\pi  \leftarrow \text{solution of the RL task } \vect{w}_i$
        \STATE Add $\pi$ to $\Pi$ and $\vect{\psi}^\pi$ to $\Psi$
        \STATE $i \leftarrow i + 1$
   \ENDWHILE
   
   \STATE{\bfseries return $\Pi, \Psi$}
\end{algorithmic}
\end{algorithm}

%%%%%%%%%%%%%%%%%%%%%%%%%%%%%%%%%%%%%%%%%%%%%%%%%%%%%%%%%%%%%%%%%%%%%%%%%%%%%%%
%%%%%%%%%%%%%%%%%%%%%%%%%%%%%%%%%%%%%%%%%%%%%%%%%%%%%%%%%%%%%%%%%%%%%%%%%%%%%%%

\end{document}